\documentclass[11pt]{article}

\usepackage[
    letterpaper,
    left=1.12in,
    right=1.12in,
    top=0.90in,
    bottom=1.00in
]{geometry}

\usepackage[T1]{fontenc}
\usepackage{hyperref}
\usepackage{url}
\usepackage{booktabs}
\usepackage{amsfonts}
\usepackage{nicefrac}
\usepackage{microtype}
\usepackage{xcolor}
\usepackage{algorithm}
\usepackage{algorithmic}
\usepackage{amsmath}
\usepackage{amssymb}
\usepackage{mathtools}
\usepackage{amsthm}
\usepackage{dsfont}
\usepackage{bm}
\usepackage{graphicx}

\usepackage[round]{natbib}

\newtheorem{assumption}{Assumption}

\newtheorem{definition}{Definition}

\newtheorem{theorem}{Theorem}
\newtheorem{corollary}{Corollary}
\newtheorem{lemma}{Lemma}

\def \A {\mathcal{A}}

\def \R {\mathcal{R}}

\def \X {\mathcal{X}}

\def \H {\mathcal{H}}

\title{\textbf{When Greedy Sampling Explores: KL-Regularized Contextual Bandits without Eluder-Dimension Dependence}}

\usepackage{authblk}

\author[1]{Zichen Wang$^{*,}$}
\author[2]{Haoyang Hong}
\author[2]{Huazheng Wang}

\affil[1]{Department of ECE and CSL, University of Illinois Urbana-Champaign}
\affil[2]{School of Electrical Engineering and Computer Science, Oregon State University}

\date{}

\begin{document}

\maketitle

\begingroup
\renewcommand{\thefootnote}{\fnsymbol{footnote}}
\footnotetext[1]{Corresponding author: \texttt{zichenw6@illinois.edu}.}
\endgroup

\maketitle

\begingroup
\renewcommand{\thefootnote}
{\fnsymbol{footnote}}
\footnotetext[2]{%
Generative AI tools assisted with language editing, presentation,
and checks of mathematical accuracy, proof consistency, notation,
and references, including suggestions for clarifying lemma assumptions
and statements. The authors independently reviewed and verified all
incorporated suggestions and take full responsibility for the final
content.%
}
\endgroup

\begin{abstract}
We study KL-regularized contextual bandits under both reward and preference feedback. While existing regret guarantees typically depend on the eluder dimension, we show that simple greedy sampling can achieve polylogarithmic regret without explicit dependence on this complexity measure. For reward feedback, we analyze a greedy algorithm that samples directly from the Gibbs policy induced by the estimated reward. We extend the result to preference feedback under both general preference and Bradley--Terry models, while also sharpening existing dimension-dependent guarantees.
Our analysis reveals a trade-off between greedy sampling and upper confidence bound-style exploration: greedy sampling enjoys stronger regret guarantees when KL regularization is sufficiently strong, whereas additional exploration yields sharper bounds as the regularization weakens.
\end{abstract}

\section{Introduction}

KL regularization has become a central ingredient in modern sequential
decision-making, appearing in both bandits
\citep{zhao2025logarithmic,zhao2025sharp,
ji2026near,ji2026optimal}
and reinforcement learning (RL) from human feedback
\citep{christiano2017deep,ouyang2022training,xiong2024iterative,
ye2024online,munos2024nash}.
By penalizing deviations from a reference policy, KL regularization prevents the learned policy from moving too far away from the reference policy.

Recent work has studied different algorithmic approaches to KL-regularized contextual bandits under different feedback models. For reward feedback (RF), \citet{zhao2025logarithmic} develop a UCB-style algorithm \citep{auer2002finite,abbasi2011improved,xu2024uppercounterfactualconfidencebounds} and achieve a logarithmic regret guarantee. More recently, under preference feedback (PF), \citet{wu2025greedy} show that greedy sampling can provide sufficient implicit exploration under KL regularization to achieve provable regret guarantees.
Their guarantees, however, still retain explicit dependence on the corresponding eluder dimension \citep{russo2013eluder,osband2014model,zhang2023mathematical}. The eluder dimension measures the complexity of sequential exploration for a given function class and can be large for rich function classes.
This raises a sharper question:

\begin{quote}
\emph{Can KL-regularized contextual bandits achieve polylogarithmic regret
without explicit dependence on the eluder dimension under both RF
and PF?}
\end{quote}

To the best of our knowledge, we establish the first eluder-dimension-independent polylogarithmic regret guarantees for KL-regularized contextual bandits under both RF and PF. These guarantees are achieved by greedy sampling, with the PF results covering both the general preference (GP) \citep{munos2024nash,ye2024online} and Bradley--Terry (BT) \citep{bradley1952rank} models. We further characterize the trade-off between greedy sampling and UCB-style exploration, showing that the relative sharpness of their regret guarantees depends on the regularization regime, a phenomenon that does not typically arise in standard contextual bandit analyses.
The main contributions are summarized as follows.

\begin{itemize}
    \item
   We establish the first eluder-dimension-independent polylogarithmic regret guarantee for KL-regularized contextual bandits under RF using greedy sampling. We then extend the result to PF under both the GP and BT models.

    \item
   For PF, we develop UCB-style algorithms and establish dimension-dependent regret guarantees, providing UCB counterparts to our greedy sampling algorithms.

    \item
    By comparing the greedy and UCB-style guarantees, we characterize a
    trade-off governed by the strength of KL regularization: greedy sampling admits a sharper regret guarantee under sufficiently strong regularization, whereas explicit optimism yields a sharper bound as the regularization weakens.
\end{itemize}

The remainder of the paper is organized as follows.
Section~\ref{related work} reviews the most relevant literature.
Section~\ref{preliminary} introduces the basic settings for contextual bandits
with RF, as well as PF under the GP and BT models.
Section~\ref{sectionbandit} studies KL-regularized contextual bandits with
RF, establishes our eluder-dimension-independent regret guarantee
for greedy sampling, and discusses the trade-off between greedy sampling and
UCB-style exploration.
Section~\ref{sectionpreference} extends the analysis of greedy sampling to PF under both the GP and BT models, and develops the corresponding UCB-style
algorithms and regret guarantees.
Section~\ref{experiment} presents numerical experimental results.
\section{Related Work}\label{related work}

\paragraph{Contextual Bandits.}

The contextual bandit literature can be broadly divided into two lines.
The first studies structured or parametric models, where the expected reward is assumed to satisfy a known structural model.
A prominent example is the linear contextual bandit, in which rewards are
linear in context--action features
\citep{li2010contextual,chu2011contextual,abbasi2011improved,
agrawal2013thompson}.
This line has also been extended beyond linear models, including
generalized linear models \citep{filippi2010parametric,li2017provably,jun2017scalable},
kernelized models \citep{valko2013finite,chowdhury2017kernelized}, and neural network models
\citep{zhou2020neural,zhang2021neural}.

A second line studies contextual bandits with general function approximation, without imposing a specific parametric form on the reward model. Early work considered general hypothesis or policy classes using explicit exploration and supervised-learning oracles \citep{langford2008epoch,agarwal2014taming}. Subsequent work developed regression-based reductions and oracle-efficient algorithms that accommodate rich, potentially nonparametric function classes \citep{foster2018practical,foster2020beyond,simchilevi2022bypassing,xu2024uppercounterfactualconfidencebounds}.
For such general function approximation settings, regret guarantees are
often characterized through complexity measures of the function class,
such as the eluder dimension
\citep{russo2013eluder,osband2014model}.
Our work falls within this general function approximation setting, but seeks regret guarantees without explicit dependence on the eluder dimension.

\paragraph{KL-Regularized Bandits and RL.}

Regularization has been extensively studied in reinforcement learning as a
mechanism for controlling policy updates and improving optimization and
exploration properties
\citep{geist2019theory,cen2022fast,zhan2023policy}.
KL regularization, in particular, is widely used in modern RL and reinforcement learning from human feedback (RLHF), where it constrains the learned policy to remain close to a reference policy \citep{ouyang2022training,rafailov2023direct}.
This practical importance has motivated a growing theoretical literature on
KL-regularized multi-armed bandits \citep{ji2026near,ji2026optimal}, contextual bandits and
reinforcement learning
\citep{zhao2025logarithmic,zhao2026fast,zhao2026sharp,hong2026online},
RLHF (human/preference feedback)
\citep{xiong2024iterative,xie2024exploratory,zhao2025sharp,wu2025greedy,wu2025offline,wu2026fdivergence}, and zero-sum Markov games \citet{nayak2025achieving}.
In particular, \citet{zhao2025logarithmic} obtain eluder-dimension-dependent logarithmic
regret for KL-regularized contextual bandits with RF using
UCB-style exploration, while \citet{wu2025greedy} establish similar guarantees for
greedy sampling under PF (both GP and BT models are considered). Besides, \citet{lee2026provably} establish polylogarithmic regret for greedy sampling under generalized bilinear preferences and strongly convex regularization with feature coverage. Their regret guarantees are not directly comparable to ours due to different model, regularization, and coverage assumptions.
\section{Preliminaries}\label{preliminary}

\subsection{Notation}

For any positive integer $n$, define $[n]:=\{1,\ldots,n\}$.
For a finite function class $\mathcal{F}$, let
$N_{\mathcal{F}}:=|\mathcal{F}|$ denote its cardinality.
We use $O(\cdot)$ and $\widetilde{O}(\cdot)$ for the standard
asymptotic notation, where $\widetilde{O}(\cdot)$ suppresses
logarithmic factors.

\subsection{KL-Regularized Contextual Bandits with Reward Feedback}

In this section, we introduce the KL-regularized contextual
bandit problem with RF. Consider a contextual
bandit problem with horizon $T$. At each round $t\in[T]$, a context $x_t\in\mathcal X$ is
independently drawn from an unknown distribution $d$ over
$\mathcal X$. After observing $x_t$, the learner selects an
action $a_t\sim\pi_t(\cdot\mid x_t)$ from an action space
$\mathcal A$, where $|\mathcal A|= K <\infty$. Conditioned on $(x_t,a_t)$, the learner observes a random
reward $r_t\in[0,1]$ that is independent of the past and
satisfies $\mathbb E[r_t\mid x_t,a_t]=R^\star(x_t,a_t)$,
where $R^\star:\mathcal X\times\mathcal A\to[0,1]$ is the
unknown expected reward function.
The learner has access to a finite candidate reward function class
$\mathcal{R}$ consisting of functions
$R:\mathcal{X}\times\mathcal{A}\to[0,1]$.
We make the following standard realizability assumption.

\begin{assumption}[Reward Realizability]
\label{assumption1}
The true expected reward function satisfies
$R^\star\in\mathcal{R}$.
\end{assumption}

For simplicity, we focus on a finite reward function class \(\mathcal{R}\); similar arguments can be extended to infinite function classes under suitable covering-number conditions \citep{zhang2023mathematical,xu2024uppercounterfactualconfidencebounds}.

\paragraph{Learning objective.} Throughout the paper, we assume that the reference policy $\pi_{\rm ref}$ has full
support over $\mathcal{A}$, i.e.,
$\pi_{\rm ref}(a\mid x)>0$ for all $(x,a)\in\mathcal{X}\times\mathcal{A}$. This guarantees that the KL regularization is well defined for any policy.
For two policies $\pi$ and $\pi_{\mathrm{ref}}$, define their
conditional KL divergence at context $x$ as
\[
    \mathrm{KL}(\pi,\pi_{\mathrm{ref}}\mid x)
    :=
    \mathbb{E}_{a\sim\pi}
    \left[
        \log
        \frac{\pi(a\mid x)}
             {\pi_{\mathrm{ref}}(a\mid x)}
    \right].
\]
Given a reference policy $\pi_{\mathrm{ref}}$ and a parameter
$\eta>0$, the KL-regularized value of a policy $\pi$ is defined as
\begin{align}
    J_{\mathrm{RF}}(\pi)
    &:=
    \mathbb{E}_{x\sim d,\,a\sim\pi}
    \left[
        R^\star(x,a)
        -
        \eta^{-1}
        \mathrm{KL}(\pi,\pi_{\mathrm{ref}}\mid x)
    \right]
    \nonumber\\
    &=
    \mathbb{E}_{x\sim d,\,a\sim\pi}
    \left[
        R^\star(x,a)
        -
        \eta^{-1}
        \log
        \frac{\pi(a\mid x)}
             {\pi_{\mathrm{ref}}(a\mid x)}
    \right].
    \label{eq:rf-objective}
\end{align}
Here, 
$\eta>0$ controls the strength of the KL regularization.
In particular, a smaller $\eta$ corresponds to stronger
regularization toward $\pi_{\mathrm{ref}}$, whereas a larger
$\eta$ allows the learned policy to deviate more substantially
from the reference policy.

Let $\pi^\star_{\text{RF}}
    :=
    \arg\max_{\pi} J_{\mathrm{RF}}(\pi)$
denote the unique optimal policy for the KL-regularized objective.
Our goal is to design a sequence of policies
$\{\pi_t\}_{t=1}^{T}$ that minimizes the cumulative regret
\begin{equation}
    \operatorname{Reg}_{\mathrm{RF}}(T)
    :=
    \sum_{t=1}^{T}
    \left(
        J_{\mathrm{RF}}(\pi^\star_{\text{RF}})
        -
        J_{\mathrm{RF}}(\pi_t)
    \right).
    \label{eq:rf-regret}
\end{equation}
 
The following lemma characterizes the unique solution of the
KL-regularized optimization problem; see, e.g.,
\citet{zhang2023mathematical}.

\begin{lemma}[Solution of the KL-Regularized Optimization Problem]
\label{lem:gibbs-policy}
For any $x\in\mathcal{X}$ and reward function $R\in\mathcal{R}$,
we have
\begin{align*}
    &\max_{\pi}
    \left\{
        \mathbb{E}_{a\sim\pi}
        [R(x,a)]
        -
        \eta^{-1}
        \mathrm{KL}(\pi,\pi_{\mathrm{ref}}\mid x)
    \right\}
=
    \eta^{-1}
    \log
    \mathbb{E}_{a\sim\pi_{\mathrm{ref}}}
    \left[
        \exp\bigl(\eta R(x,a)\bigr)
    \right].
\end{align*}
The unique maximizer is the Gibbs policy
\begin{equation}
    \pi_R(a\mid x)
    =
    \frac{
        \pi_{\mathrm{ref}}(a\mid x)
        \exp\bigl(\eta R(x,a)\bigr)
    }{
        Z_R(x)
    },
    \label{eq:gibbs-policy}
\end{equation}
where $Z_R(x)
    :=
    \sum_{a'\in\mathcal{A}}
    \pi_{\mathrm{ref}}(a'\mid x)
    \exp\bigl(\eta R(x,a')\bigr)$
is the normalizing constant.
Under Assumption~\ref{assumption1},
the optimal policy is therefore given by
$\pi^\star_{\operatorname{RF}} = \pi_{R^\star}$.
\end{lemma}

Although our main polylogarithmic regret guarantee does not explicitly
depend on the eluder dimension, we use the following uncertainty
measure and the associated eluder dimension to characterize eluder
dimension-dependent guarantees and to facilitate comparisons with
existing approaches.

\begin{definition}[Uncertainty Measure and Eluder Dimension:
Reward Feedback \citep{zhao2025logarithmic}]
\label{def:rf-eluder} For $t\geq 1$, let
$\mathcal D_t^{\operatorname{RF}}:=\{(x_i,a_i)\}_{i=1}^{t}$
denote the sequence of observed context--action pairs up to round
$t$.
For $\lambda>0$, the uncertainty of a context--action pair
$(x,a)\in\mathcal{X}\times\mathcal{A}$ with respect to the
function class $\mathcal{R}$ and the data $D^{\operatorname{RF}}_{t-1}$ is defined as
\begin{align*}
    &U_{\mathrm{RF}}
    (\lambda,x,a,\mathcal{R};\mathcal D^{\operatorname{RF}}_{t-1})
    :=
    \sup_{R_1,R_2\in\mathcal{R}}
    \frac{
        |R_1(x,a)-R_2(x,a)|
    }{
        \sqrt{
            \lambda
            +
            \sum_{i=1}^{t-1}
            \bigl(
                R_1(x_i,a_i)-R_2(x_i,a_i)
            \bigr)^2
        }
    }.
\end{align*}
The corresponding eluder dimension is defined as
\begin{align*}
    &d_{\mathrm{RF}}(\lambda,\mathcal{R},T)
    :=
    \sup_{x_{1:T},a_{1:T}}
    \sum_{t=1}^{T}
    \min
    \left\{
        1,\,
        U_{\mathrm{RF}}^2
        (\lambda,x_t,a_t,\mathcal{R};\mathcal D^{\operatorname{RF}}_{t-1})
    \right\}.
\end{align*}
\end{definition}

\subsection{KL-Regularized Contextual Bandits with Preference Feedback}

We next consider KL-regularized contextual bandits with PF. At each round $t\in[T]$, a context $x_t$ is independently
drawn from $d$. The learner selects an action pair \((a_t^1,a_t^2)\in\mathcal A^2\) and receives binary feedback \(y_t\in\{0,1\}\), where \(y_t=1\) indicates that \(a_t^1\) is preferred over \(a_t^2\), and \(y_t=0\) indicates the opposite. We consider two preference models: the GP model and the BT model.

\subsubsection{General Preference Model}

Under the GP model, preferences are characterized by an unknown
function $P^\star:\mathcal{X}\times\mathcal{A}\times\mathcal{A}
    \to[0,1]$,
where $P^\star(x,a^1,a^2)$ denotes the probability that $a^1$ is
preferred to $a^2$ under context $x$. Conditioned on $(x_t,a_t^1,a_t^2)$, the feedback $y_t$ is
generated independently of the past according to $y_t\sim
\operatorname{Ber}\!\left(
P^\star(x_t,a_t^1,a_t^2)
\right)$.
The learner has access to a finite candidate preference class
$\mathcal{P}$, where each
$P:\mathcal{X}\times\mathcal{A}\times\mathcal{A}\to[0,1]$
satisfies the standard reciprocity condition
$P(x,a^1,a^2)+P(x,a^2,a^1)=1$
for all $(x,a^1,a^2)\in\mathcal{X}\times\mathcal{A}^2$.
We make the following realizability assumption.

\begin{assumption}[GP Realizability]\label{assumption2} 
The true preference function satisfies $P^\star\in\mathcal{P}$.
\end{assumption}

\paragraph{Learning objective.}
For convenience, for any $P\in\mathcal{P}$, define $P(x,a,\pi):=
    \mathbb{E}_{a'\sim\pi}[P(x,a,a')]$
and $P(x,\pi^1,\pi^2):=
    \mathbb{E}_{a^1\sim\pi^1,\,
               a^2\sim\pi^2}
    [P(x,a^1,a^2)]$.
Under the GP model, we consider the following KL-regularized
zero-sum objective \citep{munos2024nash,ye2024online}:
\begin{align*}
    &J_{\mathrm{GP}}(\pi^1,\pi^2)
    :=
    \mathbb{E}_{x\sim d}
    \Big[
        P^\star(x,\pi^1,\pi^2)
         -
        \eta^{-1}
        \mathrm{KL}(\pi^1,\pi_{\mathrm{ref}}\mid x)
        \nonumber
        +
        \eta^{-1}
        \mathrm{KL}(\pi^2,\pi_{\mathrm{ref}}\mid x)
    \Big].
\end{align*}
Here, the first player seeks to maximize the preference value,
whereas the second player seeks to minimize it. Accordingly, the KL terms enter with opposite signs for the two players, while regularizing both policies toward the reference policy
$\pi_{\mathrm{ref}}$.

The KL-regularized zero-sum game admits the following value
\[
    J_{\mathrm{GP}}^\star
    :=
    \max_{\pi^1}\min_{\pi^2}
    J_{\mathrm{GP}}(\pi^1,\pi^2).
\]
It has been shown that this game admits a unique Nash equilibrium,
with the two equilibrium policies coinciding
\citep{munos2024nash, ye2024online}.
We denote the common equilibrium policy by
$\pi_{\mathrm{GP}}^\star$, so that $(\pi_{\mathrm{GP}}^\star,\pi_{\mathrm{GP}}^\star)$
is the unique Nash equilibrium and $J_{\mathrm{GP}}^\star
    =
    J_{\mathrm{GP}}
    (\pi_{\mathrm{GP}}^\star,\pi_{\mathrm{GP}}^\star)$.
For a learned first-player policy $\widehat{\pi}_t^1$, we measure
its suboptimality against its regularized best response: $J_{\mathrm{GP}}^\star
    -
    \min_{\pi^2}
    J_{\mathrm{GP}}(\widehat{\pi}_t^1,\pi^2)$.
Our goal is to design a sequence of first-player policies
$\{\widehat{\pi}_t^1\}_{t=1}^T$ that minimizes the cumulative regret
\[
\operatorname{Reg}_{\mathrm{GP}}(T)
    :=
    \sum_{t=1}^T
    \left(J_{\mathrm{GP}}^\star
    -
    \min_{\pi^2}
    J_{\mathrm{GP}}(\widehat{\pi}_t^1,\pi^2)\right).
\]

The following characterization of the equilibrium policy will be useful in our analysis.

\begin{lemma}[Nash Equilibrium under the GP Model \citep{wu2025greedy}]
For any $P\in\mathcal P$, the corresponding equilibrium policy
$\pi_P$ satisfies
\[
    \pi_P(a\mid x)
    =
    \frac{
        \pi_{\mathrm{ref}}(a\mid x)
        \exp\!\bigl(\eta P(x,a,\pi_P)\bigr)
    }{
        Z_P(x)
    },
\]
where $Z_P(x)
    :=
    \sum_{a'\in\mathcal{A}}
    \pi_{\mathrm{ref}}(a'\mid x)
    \exp\!\bigl(\eta P(x,a',\pi_P)\bigr)$.
In particular, the true equilibrium policy is
$\pi_{\mathrm{GP}}^\star=\pi_{P^\star}$.
\end{lemma}

We next introduce the uncertainty measure and the associated eluder
dimension for the GP model, which will be used to characterize
eluder-dimension-dependent guarantees and compare with existing methods.

\begin{definition}[Uncertainty Measure and Eluder Dimension:
General Preference Model \citep{wu2025greedy}]For $t \ge 1$, let $\mathcal D_t^{\mathrm{GP}}
    :=
    \{(x_i,a_i^1,a_i^2)\}_{i=1}^t$.
For $\lambda>0$, define
\begin{align*}
&U_{\mathrm{GP}}
(\lambda,x,a^1,a^2,\mathcal{P};\mathcal D_{t-1}^{\mathrm{GP}})
:=
\sup_{P_1,P_2\in\mathcal{P}}
\frac{
|P_1(x,a^1,a^2)-P_2(x,a^1,a^2)|
}{
\sqrt{
\lambda+
\sum_{i=1}^{t-1}
\bigl(
P_1(x_i,a_i^1,a_i^2)
-
P_2(x_i,a_i^1,a_i^2)
\bigr)^2
}
}.
\end{align*}
The corresponding eluder dimension is
\begin{align*}
&d_{\mathrm{GP}}(\lambda,\mathcal{P},T)
:=
\sup_{x_{1:T},a_{1:T}^1,a_{1:T}^2}
\sum_{t=1}^T
\min\left\{
1,\,
U_{\mathrm{GP}}^2
(\lambda,x_t,a_t^1,a_t^2,\mathcal{P};
\mathcal D_{t-1}^{\mathrm{GP}})
\right\}.
\end{align*}
\end{definition}

\subsubsection{Bradley--Terry Model}

Under the BT model, preferences are induced by the latent reward
function $R^\star:\mathcal{X}\times\mathcal{A}\to[0,1]$ through $P^\star(x,a^1,a^2)
=
\sigma\!\left(
R^\star(x,a^1)-R^\star(x,a^2)
\right)$,
where $\sigma(z):=(1+e^{-z})^{-1}$. Conditioned on $(x_t,a_t^1,a_t^2)$, the feedback $y_t$ is
independent of the past and follows
$y_t\sim
\operatorname{Ber}\!\left(
\sigma\!\left(
R^\star(x_t,a_t^1)-R^\star(x_t,a_t^2)
\right)
\right)$.
As in the RF setting, we assume that Assumption~\ref{assumption1} holds.

\textbf{Learning objective.}
Since the BT model is induced by the latent reward $R^\star$, we
evaluate policies using the same KL-regularized reward objective:
$J_{\mathrm{BT}}(\pi):=J_{\mathrm{RF}}(\pi)$.
Consequently, $\pi_{\mathrm{BT}}^\star = \pi^\star_{\text{RF}}$ in both the RF and BT settings.
Our goal is to design a sequence of policies
$\{\widehat{\pi}_t^1\}_{t=1}^T$ that minimizes the cumulative regret
\[
\operatorname{Reg}_{\mathrm{BT}}(T)
:=
\sum_{t=1}^T
\left(
J_{\mathrm{BT}}(\pi^\star_{\text{BT}})
-
J_{\mathrm{BT}}(\widehat{\pi}_t^1)
\right).
\]
Thus, the BT and reward-feedback settings share the same policy
objective, but differ in the observed feedback and consequently in
how $R^\star$ is estimated.

We similarly define the uncertainty measure and eluder dimension under the BT model.

\begin{definition}[Uncertainty Measure and Eluder Dimension:
Bradley--Terry Model \citep{wu2025greedy}] For $t \ge 1$, let $\mathcal{D}_t^{\mathrm{BT}}
    :=
    \{(x_i,a_i^1,a_i^2)\}_{i=1}^t$.
For $\lambda>0$, define
\begin{align*}
&U_{\mathrm{BT}}
(\lambda,x,a^1,a^2,\mathcal{R};\mathcal{D}_{t-1}^{\mathrm{BT}})
:=
\sup_{R_1,R_2\in\mathcal{R}}
\frac{
\left|
(R_1-R_2)(x,a^1)
-
(R_1-R_2)(x,a^2)
\right|
}{
\sqrt{
\lambda+
\sum_{i=1}^{t-1}
\left[
(R_1-R_2)(x_i,a_i^1)
-
(R_1-R_2)(x_i,a_i^2)
\right]^2
}
}.
\end{align*}
The corresponding eluder dimension is
\begin{align*}
&d_{\mathrm{BT}}(\lambda,\mathcal{R},T)
:=
\sup_{x_{1:T},a_{1:T}^1,a_{1:T}^2}
\sum_{t=1}^T
\min\left\{
1,\,
U_{\mathrm{BT}}^2
(\lambda,x_t,a_t^1,a_t^2,\mathcal{R};
\mathcal{D}_{t-1}^{\mathrm{BT}})
\right\}.
\end{align*}
\end{definition}
\section{KL-Regularized Contextual Bandits with Reward Feedback}\label{sectionbandit}

\begin{algorithm}[t] 
\caption{\texttt{RF-GS}} 
\label{alg1} 
\begin{algorithmic}[1] 
\STATE \textbf{Input:} $\mathcal{R}, \eta, \pi_{\mathrm{ref}}$
\STATE Choose any $\widehat R_1\in\mathcal R$ and define
\[
    \pi_1(\cdot\mid\cdot)
    \propto
    \pi_{\rm ref}(\cdot\mid\cdot)
    \exp\!\left(\eta\widehat R_1(\cdot,\cdot)\right).
\]
\STATE Observe $x_1\sim d$, take action
$a_1\sim\pi_1(\cdot\mid x_1)$, and observe reward $r_1$
\FOR{round $t=2,\ldots,T$}
    \STATE Observe context $x_t\sim d$
    \STATE Compute the LS estimator
    \[
        \widehat R_t
        \in
        \arg\min_{R\in\mathcal R}
        \sum_{i=1}^{t-1}
        \bigl(R(x_i,a_i)-r_i\bigr)^2
    \]
    \STATE Compute the Gibbs policy  by Eq.~\eqref{gibbsalg1}
    \STATE Take action $a_t\sim\pi_t(\cdot\mid x_t)$
    and observe $r_t$
\ENDFOR
\end{algorithmic} 
\end{algorithm}

We propose \texttt{RF-GS} (Reward-Feedback Greedy Sampling), a simple greedy algorithm for KL-regularized contextual
bandits under RF. The pseudocode is provided in
Algorithm~\ref{alg1}.

\paragraph{\texttt{RF-GS} Algorithm.}
At the beginning of each round \(t\), the algorithm observes the context
\(x_t\). Using the historical observations
$\H^{\text{RF}}_{t-1}=\{(x_i,a_i,r_i)\}_{i=1}^{t-1}$, it computes the reward
estimator \(\widehat R_t\) by solving a least-squares (LS) regression
problem over the function class \(\mathcal R\). It then directly
constructs the Gibbs policy induced by the estimated reward,
\begin{align}\label{gibbsalg1}
    \pi_t(a\mid x)
    \propto
    \pi_{\rm ref}(a\mid x)
    \exp\!\left(\eta\widehat R_t(x,a)\right).
\end{align}
This direct plug-in step is what we refer to as \emph{greedy sampling}:
the policy is constructed solely from the current reward estimate,
without any explicit uncertainty-dependent exploration bonus.
In contrast, UCB-based methods such as \texttt{K-UCB}
\citep{zhao2025logarithmic} augment the reward
estimate with an exploration bonus $b_t(x,a)$ (see the discussion in our Appendix) that favors actions with greater uncertainty.
Finally, \texttt{RF-GS} samples an action
\(a_t\sim\pi_t(\cdot\mid x_t)\) and observes the corresponding reward
\(r_t\).

The following theorem shows that, despite its simplicity,
\texttt{RF-GS} can achieve polylogarithmic regret without any explicit
dependence on the eluder dimension.

\begin{theorem}[Eluder Dimension-Independent Regret Bound for \texttt{RF-GS}]
\label{theorem1}
Suppose Assumption~\ref{assumption1} holds. Then, for any
$\delta\in(0,1)$ and $T \geq 2$, with probability at least $1-\delta$, the regret of
\texttt{RF-GS} after $T$ rounds satisfies
\[
\operatorname{Reg}_{\mathrm{RF}}(T)
=
O\!\left(
\eta e^{2\eta}
\log T
\log\frac{N_{\mathcal R}T}{\delta}
\right).
\]
\end{theorem}

\paragraph{Proof Sketch of Theorem~\ref{theorem1}.}
The proof proceeds in three steps. First, we upper bound the
instantaneous regret by an expected squared prediction error $S_t$.
Second, a uniform prediction-error bound, together with a
likelihood-ratio bound between Gibbs policies and an averaging
argument over past rounds, yields
$S_t=\widetilde{O}(e^{2\eta}/t)$.
Finally, summing this bound over $t$ gives the logarithmic
dependence on the horizon.

\paragraph{Step 1: Regret decomposition.}
By Lemma~\ref{lemma3} in the Appendix, the instantaneous regret for each
$t\in[T]$ can be upper bounded by $J_{\mathrm{RF}}(\pi^\star_{\text{RF}})-J_{\mathrm{RF}}(\pi_t)
\le
\eta S_t$,
where
\[
S_t
:=
\mathbb E_{x\sim d,\;a\sim\pi_t'}
\left[
\left(
\widehat R_t(x,a)-R^\star(x,a)
\right)^2
\right],
\]
and $\pi_t'$ denotes the Gibbs policy induced by
$R_t'
=
\gamma_t\widehat R_t
+
(1-\gamma_t)R^\star$
for some $\gamma_t\in[0,1]$.
Summing the above inequality over $t\in[T]$ yields
\begin{align}
\label{decomposition}
\operatorname{Reg}_{\rm RF}(T)
&=
\sum_{t=1}^{T}
\Bigl(
J_{\rm RF}(\pi^\star_{\text{RF}})-J_{\rm RF}(\pi_t)
\Bigr)
\le
\eta\sum_{t=1}^{T}S_t.
\end{align}

\paragraph{Step 2: Upper bound on $S_t$.}
We next apply the following uniform convergence result, which controls
the cumulative squared prediction error under the data-generating
policies.

\begin{lemma}[Uniform Prediction Error Bound, Reward Feedback]
\label{uniform convergence} Suppose Assumption~\ref{assumption1} holds.
Let $\widehat R_t$ be the LS estimator over $\mathcal R$
constructed from the samples
$\{(x_i,a_i,r_i)\}_{i=1}^{t-1}$,
where $x_i\sim d$ and
$a_i\sim\pi_i(\cdot\mid x_i)$ for each $i$.
Then, for any such policy sequence
$\{\pi_i\}_{i\ge1}$ and any $\delta\in(0,1)$,
with probability at least $1-\delta$, the following holds
simultaneously for all $t=2,\ldots,T$:
\begin{align}
\begin{split}
\label{sumbound}
&\sum_{i=1}^{t-1}
\mathbb E_{x\sim d,\;a\sim\pi_i}
\left[
\left(
R^\star(x,a)-\widehat R_t(x,a)
\right)^2
\right]
\le
72\log\frac{2N_{\mathcal R}T^3}{\delta}.
\end{split}
\end{align}
\end{lemma}

We now relate the expectation in $S_t$, which is taken under the
intermediate Gibbs policy $\pi_t'$, to the prediction errors under the
data-generating policies $\{\pi_i\}_{i=1}^{t-1}$.
For $t=1$, we trivially have $S_1\le1$.
For any $t\ge2$ and $i<t$, since both $\pi_t'$ and $\pi_i$ are Gibbs
policies induced by $[0,1]$-valued reward functions, their likelihood
ratio satisfies
\[
\frac{\pi_t'(a\mid x)}{\pi_i(a\mid x)}
\le
e^{2\eta},
\
\forall (x,a).
\]
Therefore,
\[
S_t
\le
e^{2\eta}
\mathbb E_{x\sim d,\;a\sim\pi_i}
\left[
\left(
R^\star(x,a)-\widehat R_t(x,a)
\right)^2
\right].
\]
Averaging the above inequality over $i=1,\ldots,t-1$ and applying
Eq.~\eqref{sumbound} yields
\begin{align}
\label{Sbound}
S_t
\le
\frac{72e^{2\eta}}{t-1}
\log\frac{2N_{\mathcal R}T^3}{\delta},
\qquad
\forall t=2,\ldots,T.
\end{align}

\paragraph{Step 3: Summation over time.}
Summing the bound on $S_t$ over $t$ and using
$\sum_{t=2}^{T}(t-1)^{-1}\le1+\log T$ gives
\begin{align}
\label{SumSbound}
\sum_{t=1}^{T}S_t
\le
1+
72e^{2\eta}(1+\log T)
\log\frac{2N_{\mathcal R}T^3}{\delta}.
\end{align}
Finally, substituting Eq.~\eqref{SumSbound} into
Eq.~\eqref{decomposition} yields
\[
\operatorname{Reg}_{\mathrm{RF}}(T)
=
O\!\left(
\eta e^{2\eta}
\log T
\log\frac{N_{\mathcal R}T}{\delta}
\right).
\]
The complete proof is provided in the Appendix.

\paragraph{Eluder-Dimension-Dependent Regret Bound for \texttt{RF-GS}.}
We can also derive an eluder-dimension-dependent regret bound for
\texttt{RF-GS} using a standard uncertainty-based analysis.

\begin{corollary}[Dimension-Dependent Regret Bound for \texttt{RF-GS}]
\label{corollary1}
Fix any $\lambda>0$ satisfying
$\lambda\le 8\log\frac{2N_{\mathcal R}T}{\delta}$.
Then, under Assumption~\ref{assumption1}, for any $\delta\in(0,1)$, with probability
at least $1-\delta$, the regret of \texttt{RF-GS} satisfies
\[
    \operatorname{Reg}_{\rm RF}(T)
    =
    O\!\left(
        \eta e^{2\eta}
        \left(
            d_{\rm RF}(\lambda,\mathcal R,T)
            +\log\frac{1}{\delta}
        \right)
        \log\frac{N_{\mathcal R}T}{\delta}
    \right).
\]
\end{corollary}

The proof differs from that of Theorem~\ref{theorem1} only in the
control of $\sum_{t=1}^T S_t$.
On the high-probability event from Lemma~\ref{lemmac1} in the Appendix, for any $(x,a)\in\X\times\A$,
\begin{align*}
&\bigl(\widehat R_t(x,a)-R^\star(x,a)\bigr)^2
\le
16\log\frac{2N_{\mathcal R}T}{\delta}
\min\!\left\{
1,
U_{\rm RF}^2(\lambda,x,a,\mathcal R;\mathcal D^{\operatorname{RF}}_{t-1})
\right\}.
\end{align*}
Using the likelihood-ratio bound
$\pi_t'(a\mid x)/\pi_t(a\mid x)\le e^{2\eta}$, together with the
definition of $d_{\rm RF}(\lambda,\mathcal R,T)$ and a standard
predictable-to-realized concentration argument, we obtain
\[
\sum_{t=1}^T S_t
=
O\!\left(
e^{2\eta}
\left(
d_{\rm RF}(\lambda,\mathcal R,T)
+\log\frac{1}{\delta}
\right)
\log\frac{N_{\mathcal R}T}{\delta}
\right).
\]
Combining this with the regret decomposition in
Eq.~\eqref{decomposition} gives the stated result.
The complete proof is provided in the Appendix.

Compared with Corollary~\ref{corollary1}, Theorem~\ref{theorem1}
replaces the uncertainty-complexity term
$d_{\rm RF}(\lambda,\mathcal R,T)+\log(1/\delta)$
with the logarithmic factor $\log T$, while retaining the same
dependence on $\eta$.

\paragraph{Trade-offs Between Greedy Sampling and UCB Exploration.}
We next compare the available regret guarantees of \texttt{RF-GS}
and \texttt{K-UCB} under different strengths of KL regularization.
When $e^{2\eta}\log T
\lesssim d_{\mathrm{RF}}(\lambda,\mathcal R,T)$,
the regret bound of \texttt{RF-GS}, $O\!\left(
\eta e^{2\eta}\log T
\log\frac{N_{\mathcal R}T}{\delta}
\right)$
(Theorem~\ref{theorem1}), is sharper than the dimension-dependent
bound of \texttt{K-UCB}, $O\!\left(
\eta d_{\rm RF}(\lambda,\mathcal R,T)
\log\frac{N_{\mathcal R}T}{\delta}
\right)$
(see Theorem~4.1 in
\citet{zhao2025logarithmic}).
Intuitively, for small $\eta$, the KL regularization is strong and the
Gibbs policy $\pi^\star_{\text{RF}}(a\mid x)
=
\pi_{\rm ref}(a\mid x)\exp(\eta R^\star(x,a))/Z_{R^\star}(x)$
remains relatively close to the reference policy. In this regime,
the stochasticity inherited from the reference policy can provide
effective exploration without an explicit optimism bonus.

For larger $\eta$, the $e^{2\eta}$ factor weakens the \texttt{RF-GS}
guarantee, and \texttt{K-UCB} admits the sharper bound $O\!\left(
\eta d_{\rm RF}(\lambda,\mathcal R,T)
\log\frac{N_{\mathcal R}T}{\delta}
\right)$.
Intuitively, as $\eta$ increases, the effect of KL regularization weakens and
the problem approaches the standard contextual bandit setting, where the
implicit exploration of greedy sampling is no longer sufficient and additional exploration can provide a sharper regret guarantee.

\texttt{K-UCB} also admits the $\eta$-independent bound
$O\!\left(
\sqrt{
T\left(
d_{\mathrm{RF}}(\lambda,\mathcal R,T)
+\log\frac{1}{\delta}
\right)
\log\frac{N_{\mathcal R}T}{\delta}
}
\right)$,
as established in Corollary~\ref{corollary2} in the Appendix. This bound becomes
the sharper \texttt{K-UCB} guarantee when $\eta$ is sufficiently large. Intuitively, for very large $\eta$, the KL regularization becomes negligible
and the problem approaches a standard contextual bandit problem with a \(\sqrt T\)-type regret rate, so the regret
bound naturally recovers the usual contextual-bandit scaling.
The available \texttt{K-UCB} guarantee can be summarized as $\widetilde O\!\left(
\min\left\{
\eta d_{\mathrm{RF}}(\lambda,\mathcal R,T),
\sqrt{d_{\mathrm{RF}}(\lambda,\mathcal R,T)T}
\right\}
\right)$.
Thus, as the KL regularization becomes weaker ($\eta$ becomes larger), the comparison exhibits
a clear transition: the eluder-dimension-independent guarantee of
\texttt{RF-GS} is sharper in the strongly regularized regime, whereas
\texttt{K-UCB} becomes sharper once the exponential dependence on
$\eta$ dominates. As $\eta$ increases further, the regret guarantee of
\texttt{K-UCB} eventually saturates at the standard $\sqrt{T}$-type
rate, which is independent of $\eta$.

\section{KL-Regularized Contextual Bandits with Preference Feedback}
\label{sectionpreference}

\begin{algorithm}[t]
\caption{\texttt{ORLHF-GS} \citep{wu2025greedy}}
\label{alg2}
\begin{algorithmic}[1]
\STATE \textbf{Input:} $\eta$, $\pi_{\mathrm{ref}}$, $\mathcal P$, $\mathcal R$

\STATE \textbf{Initialize:}
Choose any $\widehat P_1\in\mathcal P$ and
$\widehat R_1\in\mathcal R$

\FOR{$t=1,\ldots,T$}

    \STATE Observe context $x_t\sim d$

    \IF{\textbf{GP model}}
        \IF{$t\ge 2$}
            \STATE Compute the MLE
            \[
            \widehat P_t
            \in
            \arg\max_{P\in\mathcal P}
            \sum_{i=1}^{t-1}
            \Big[
            y_i\log P(x_i,a_i^1,a_i^2)
            +
            (1-y_i)\log P(x_i,a_i^2,a_i^1)
            \Big]
            \]
        \ENDIF
        \STATE Set $\widehat\pi_t^1$ to the NE policy associated with
        $\widehat P_t$

    \ELSIF{\textbf{BT model}}
        \IF{$t\ge 2$}
            \STATE Compute the MLE
            \[
            \widehat R_t
            \in
            \arg\max_{R\in\mathcal R}
            \sum_{i=1}^{t-1}
            \Big[
            y_i
            \log\sigma\!\big(
            R(x_i,a_i^1)-R(x_i,a_i^2)
            \big)
            +
            (1-y_i)
            \log\sigma\!\big(
            R(x_i,a_i^2)-R(x_i,a_i^1)
            \big)
            \Big]
            \]
        \ENDIF
        \STATE Construct the Gibbs policy
        \[
        \widehat\pi_t^1(a\mid x)
        \propto
        \pi_{\mathrm{ref}}(a\mid x)
        \exp\!\left(
        \eta\widehat R_t(x,a)
        \right)
        \]
    \ENDIF

    \STATE Sample
    $a_t^1\sim\widehat\pi_t^1(\cdot\mid x_t)$ and
    $a_t^2\sim\pi_{\mathrm{ref}}(\cdot\mid x_t)$

    \STATE Observe preference feedback $y_t$

\ENDFOR
\end{algorithmic}
\end{algorithm}

In this section, we show that the greedy algorithm
\texttt{ORLHF-GS} (Online RLHF with Greedy Sampling), proposed by
\citet{wu2025greedy}, also achieves an
eluder-dimension-independent regret bound.
The detailed procedure is summarized in Algorithm~\ref{alg2}.

\paragraph{\texttt{ORLHF-GS} Algorithm.}
\texttt{ORLHF-GS} applies greedy sampling in the PF setting. Compared with the RF setting,
the learner samples an action pair $(a_t^1,a_t^2)$ rather than a single
action, where $a_t^1$ is drawn from the current learned policy $\widehat{\pi}_t^1$ and
$a_t^2$ from the reference policy.
At each round, the learner updates the preference or reward model using previously collected observations via maximum likelihood estimation (MLE), and then constructs the corresponding greedy policy.

\subsection{General Preference Model}

The following theorem establishes two complementary regret
guarantees for \texttt{ORLHF-GS} under the GP model.

\begin{theorem}[Regret Bounds under the General Preference Model for \texttt{ORLHF-GS}]\label{theorem2}
Under Assumption~\ref{assumption2}, for any $\delta\in(0,1)$ and $T \geq 2$, with probability
at least $1-\delta$,
\[
\operatorname{Reg}_{\operatorname{GP}}(T)
=
O\!\left(
\eta e^{3\eta}\log T
\log\frac{N_{\mathcal P}T}{\delta}
\right).
\]
Moreover, for any $\lambda>0$ satisfying
$\lambda\le \log(2N_{\mathcal P}T/\delta)$, with probability
at least $1-\delta$,
\begin{align*}
\operatorname{Reg}_{\operatorname{GP}}(T)
=
O\!\Bigg(
&
\eta e^{\eta}
\left(
d_{\operatorname{GP}}(\lambda,\mathcal P,T)
+
\log\frac{1}{\delta}
\right) 
\log\frac{N_{\mathcal P}T}{\delta}
\Bigg).
\end{align*}
\end{theorem}

Theorem~\ref{theorem2} improves upon and complements the
dimension-dependent analysis of \citet{wu2025greedy} in two respects. While their main theorem suppresses the
detailed dependence on $\eta$, the full bound derived in Appendix~D.1
(p.~30) is $O\!\left(
\left(
\eta e^{3\eta}+\eta^3 e^{9\eta}
\right)
d_{\operatorname{GP}}(\lambda,\mathcal P,T)
\log\frac{N_{\mathcal P}T}{\delta}
\right)$.
In contrast, Theorem~\ref{theorem2} establishes the
eluder-dimension-independent bound $O\!\left(
\eta e^{3\eta}
\log T
\log\frac{N_{\mathcal P}T}{\delta}
\right)$,
which removes both the explicit dependence on
$d_{\operatorname{GP}}(\lambda,\mathcal P,T)$ and the
$\eta^3e^{9\eta}$ term. Moreover, even under an
eluder-dimension-based analysis, our bound improves the
dependence on the regularization parameter from
$\eta e^{3\eta}+\eta^3e^{9\eta}$ to $\eta e^\eta$.

The key technical ingredient behind these improvements is the refined
instantaneous regret decomposition in
Lemma~\ref{lem:gp_regret_decomposition} in the Appendix. As in
\citet{wu2025greedy}, a central difficulty is the mismatch between the learned equilibrium policy $\widehat{\pi}^1_t$ and its true regularized best response $\widetilde{\pi}_t^2
:=
\arg\min_{\pi^2}
J_{\operatorname{GP}}
\left(
\widehat{\pi}_t^1,\pi^2
\right)$. Our analysis exploits the structure of the KL-regularized
game more directly: it expresses the instantaneous regret in terms of
the KL divergence between these two policies and then uses the
curvature of the KL regularizer, together with the Gibbs-policy
boundedness relative to the reference policy, to control this
divergence by the squared preference-prediction error. This yields the
sharper instantaneous bound
\begin{align}
\begin{split}
\nonumber
J_{\operatorname{GP}}^\star
    -
    J_{\operatorname{GP}}
    \left(
        \widehat\pi_t^1,
        \widetilde\pi_t^2
    \right)
\le 
2\eta e^\eta
\mathbb E_{
x\sim d,\,
a^1\sim\widehat\pi_t^1,\,
a^2\sim\pi_{\rm ref}
}
\left[
\bigl(
P^\star(x,a^1,a^2)-\widehat P_t(x,a^1,a^2)
\bigr)^2
\right].
\end{split}
\end{align}
This decomposition is the common ingredient behind both the
eluder-dimension-independent guarantee and the sharper
dimension-dependent bound.

For completeness, we also develop a UCB-style exploration
algorithm for the GP model, termed \texttt{GP-UCB}.
The algorithm and its analysis are deferred to the Appendix.
The following corollary summarizes its regret guarantee.

\begin{corollary}[Regret Bound for \texttt{GP-UCB}]
\label{cor:gp_ucb}
Under Assumption~\ref{assumption2}, for any $\delta\in(0,1)$ and any $\lambda>0$
satisfying $\lambda
    \le
    \log\frac{2N_{\mathcal P}T}{\delta}$,
\texttt{GP-UCB} satisfies, with
probability at least $1-\delta$,
\begin{align}
    \operatorname{Reg}_{\operatorname{GP}}(T)
    =
    \widetilde{O}\!\left(
        \min\left\{
        \eta d_{\operatorname{GP}}(\lambda,\mathcal P,T),
        \sqrt{
        d_{\operatorname{GP}}(\lambda,\mathcal P,T)T
        }
        \right\}
    \right).\nonumber
\end{align}
\end{corollary}

Corollary~\ref{cor:gp_ucb} reveals a trade-off similar to that in the
RF setting. In the strongly regularized regime, when $e^{3\eta}\log T
\lesssim
d_{\mathrm{GP}}(\lambda,\mathcal P,T)$,
the eluder-dimension-independent guarantee of \texttt{ORLHF-GS}
 is sharper than the
dimension-dependent \texttt{GP-UCB} guarantee.
As $\eta$ increases, the exponential dependence in the greedy
bound eventually dominates, and UCB-style exploration
provides the sharper guarantee. For sufficiently large
$\eta$, the available \texttt{GP-UCB} bound further saturates at
the standard
$\widetilde O(\sqrt{d_{\operatorname{GP}}T})$
rate, which is independent of $\eta$.

\subsection{Bradley--Terry Model}

We next turn to the BT model. Since the BT
model is induced by a latent reward function, its analysis is closely
connected to the reward-feedback setting. 

\begin{theorem}[Regret Bound under the Bradley--Terry Model for \texttt{ORLHF-GS}]
\label{theorem3}
Under Assumption~\ref{assumption1}, for any $\delta\in(0,1)$ and $T \geq 2$, with probability at
least $1-\delta$, \texttt{ORLHF-GS} under the BT model satisfies
\[
\operatorname{Reg}_{\operatorname{BT}}(T)
=
O\!\left(
\eta e^{2\eta}
\log T
\log\frac{N_{\mathcal R}T}{\delta}
\right).
\]
\end{theorem}

Theorem~\ref{theorem3} strengthens the existing BT analysis of
\citet{wu2025greedy} by removing the explicit dependence
on the BT eluder dimension. Tracking the explicit $\eta$-dependence
in their proof yields a dimension-dependent regret bound of order $\widetilde{O}\!\left(
\eta e^{2\eta}
d_{\operatorname{BT}}(\lambda,\mathcal R,T)
\right)$ (see their Appendix~D.2
(p.~31)).
In contrast, Theorem~\ref{theorem3} replaces
$d_{\operatorname{BT}}(\lambda,\mathcal R,T)$ by only a logarithmic
dependence on $T$, while retaining the same dependence on $\eta$.

For completeness, we also consider the corresponding UCB-style
exploration rule for the BT model, termed \texttt{BT-UCB};
its construction and analysis are deferred to the Appendix.

\begin{corollary}[Regret Bound for \texttt{BT-UCB}]
\label{cor:bt_ucb}
Under Assumption~\ref{assumption1}, for any $\delta\in(0,1)$ and any $\lambda>0$
satisfying $\lambda
    \le
    4e^2\log\frac{2N_{\mathcal R}T}{\delta}$,
\texttt{BT-UCB} satisfies,
with probability at least $1-\delta$,
\begin{align}
    \operatorname{Reg}_{\operatorname{BT}}(T)
    =
    \widetilde{O}\left(
        \min\left\{
        \eta d_{\operatorname{BT}}(\lambda,\mathcal R,T),
        \sqrt{
        d_{\operatorname{BT}}(\lambda,\mathcal R,T)T
        }
        \right\}
    \right).
\nonumber
\end{align}
\end{corollary}

The eluder-dimension-independent regret bound for \texttt{ORLHF-GS} under the BT model matches the order of its RF counterpart in Theorem \ref{theorem1}. Likewise, the \texttt{BT-UCB} guarantee has the same form as the \texttt{K-UCB} guarantee discussed in Section \ref{sectionbandit}, with the corresponding eluder dimension replaced by \(d_{\mathrm{BT}}(\lambda,\mathcal R,T)\). Thus, the comparison between the greedy and UCB regret guarantees exhibits the same dependence on \(\eta\) in the BT and RF settings.
\section{Experiments}\label{experiment}

In this section, we use a simple synthetic experiment to illustrate the greedy--UCB trade-off discussed in Section \ref{sectionbandit}.

\begin{figure}[t]
    \centering
    \includegraphics[width=0.6\linewidth]{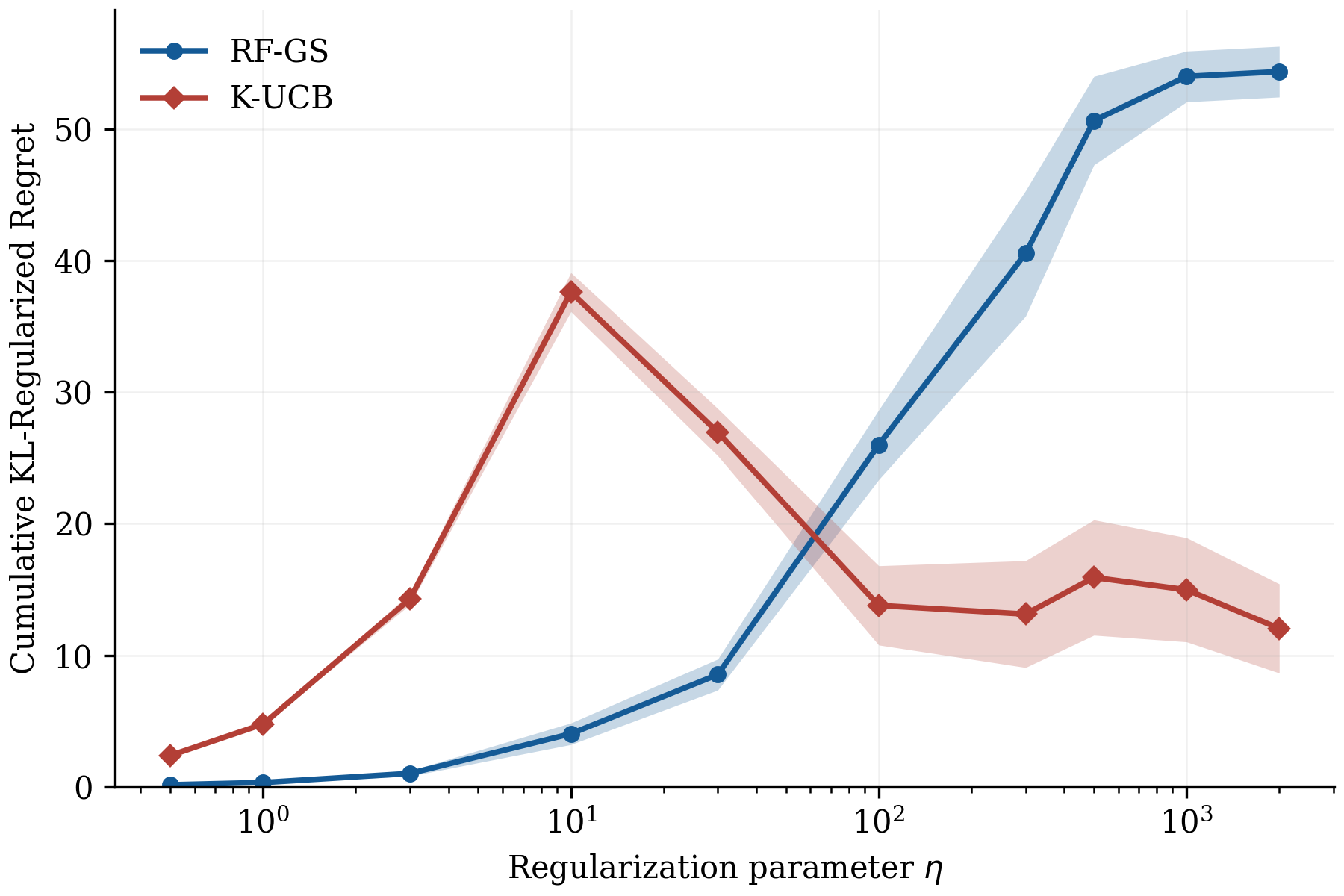}
    \caption{
Mean cumulative KL-regularized regret of \texttt{RF-GS} and
\texttt{K-UCB} across different values of $\eta$. 
}
    \label{fig:large-regret}
\end{figure}

\paragraph{Experimental setup.}
We consider a synthetic RF contextual bandit under KL regularization with three contexts (i.e., $|\X| = 3$),
three actions (i.e., $|\A| = 3$) and a function class with five functions (i.e., $|\R| = 5$). We compare the \texttt{RF-GS} (Algorithm \ref{alg1}) with \texttt{K-UCB} \citep{zhao2025logarithmic}. Each algorithm is run for $T=5000$ rounds, and the results are
averaged over $10$ independent random seeds, with
shaded regions indicating one standard deviation. To investigate the effect of KL
regularization, we vary the regularization parameter over
\begin{align}
\begin{split}
\nonumber
\eta \in
\{&0.5,1,3,10,30,100,300,500,1000,2000\}.
\end{split}
\end{align}

\paragraph{Results.}
Fig.~\ref{fig:large-regret} reveals a clear transition between the two
exploration strategies as the strength of KL regularization varies. For small
and moderate $\eta$, \texttt{RF-GS} achieves substantially lower regret than
\texttt{K-UCB}, consistent with our theory that strong KL regularization provides
sufficient implicit exploration for greedy sampling. As $\eta$ increases, the KL regularization weakens and the Gibbs policy becomes more concentrated. Consequently, \texttt{RF-GS} performs less exploration and may remain concentrated on suboptimal actions. In contrast, \texttt{K-UCB} explicitly explores uncertain actions through its exploration bonus, which becomes increasingly beneficial in this regime. Consequently, \texttt{K-UCB} eventually
outperforms \texttt{RF-GS}. Overall, the empirical results closely match the
theoretical tradeoff established in Section \ref{sectionbandit}.

\section{Conclusion}
We studied KL-regularized contextual bandits under both RF and PF, including the general preference and Bradley--Terry models. We showed that, under sufficiently strong KL regularization, greedy sampling can provide sufficient implicit exploration without relying on additional exploration bonuses. This leads to polylogarithmic regret guarantees that have no explicit dependence on the eluder dimension, contrasting with standard UCB-style guarantees whose complexity typically scales with this dimension. Our results also reveal a trade-off between greedy sampling and optimism-based exploration. When the KL regularization is strong, the learned policy remains sufficiently close to the stochastic reference policy, allowing its inherent randomness to drive exploration and making greedy sampling favorable. As the regularization weakens and the policy becomes more concentrated, this implicit exploration diminishes, and additional exploration can provide a sharper regret guarantee. These findings clarify when a simple greedy strategy can be theoretically effective and how the strength of KL regularization determines the preferred exploration mechanism.

\bibliographystyle{apalike}
\bibliography{reference}

\clearpage
\onecolumn
\newpage

\appendix

\newgeometry{
    left=0.85in,
    right=0.85in,
    top=0.85in,
    bottom=0.90in
}

\section{Notations}

We summarize the main notation used throughout the paper in the following
Table.

\begin{table}[h]
\centering
\small
\begin{tabular}{p{0.27\linewidth} p{0.66\linewidth}}
\hline
\textbf{Notation} & \textbf{Description} \\
\hline

$\mathcal X$
& Context space. \\

$\A$ & Action space.\\

$x_t$
& Context at round $t$. \\

$a_t$ & Action at round $t$ under RF.\\

$a_t^1,a_t^2$
& Pair of actions selected at round $t$ under PF. \\

$r_t$
& Reward observed at round $t$. \\

$y_t$
& Binary PF observed at round $t$. \\

$d$
& Distribution over contexts. \\

$\pi_{\rm ref}$
& Reference policy with full support. \\

$\eta$
& KL-regularization parameter; smaller $\eta$ corresponds to stronger regularization. \\

$\mathcal R,\mathcal P$
& Candidate reward and preference function classes, with
$N_{\mathcal R}=|\mathcal R|$ and $N_{\mathcal P}=|\mathcal P|$. \\

$R^\star,P^\star$
& True reward function and true preference function. \\

$\widehat R_t,\widehat P_t$
& Reward and preference estimators at round $t$. \\

$J_{\rm RF},J_{\rm GP},J_{\rm BT}$
& KL-regularized objectives under the RF, GP, and BT settings. \\

$\pi^\star_{\text{RF}},\pi_{\rm GP}^\star,\pi_{\text{BT}}^\star$
& Optimal RF policy, equilibrium GP policy, and optimal BT policy. \\

$\pi_R$
& Gibbs policy induced by the reward function $R$. \\

$\pi_P$
& Symmetric Nash-equilibrium policy induced by the preference model $P$. \\

$\operatorname{Reg}_{\rm RF}(T),
\operatorname{Reg}_{\rm GP}(T),
\operatorname{Reg}_{\rm BT}(T)$
& Cumulative regret under the RF, GP, and BT settings. \\

$U_{\rm RF},U_{\rm GP},U_{\rm BT}$
& Uncertainty measures under the RF, GP, and BT settings. \\

$d_{\rm RF},d_{\rm GP},d_{\rm BT}$
& Corresponding eluder dimensions. \\

\hline
\end{tabular}
\end{table}

\newpage

\section{Useful Properties of Gibbs Policies}\label{propertysec}

For any reward function $R:\mathcal X\times\mathcal A\to\mathbb [0,1]$,
the Gibbs policy induced by $R$ is defined as
\[
\pi_R(a|x)
=
\frac{\pi_\text{ref}(a\mid x)\exp(\eta R(x,a))}
{Z_R(x)},
\]
where $Z_R(x)
=
\sum_{a'}
\pi_{\rm ref}(a'\mid x)\exp\left(\eta R(x,a')\right)$ is the normalizing constant. It satisfies the following properties:

\begin{itemize}
    \item \textbf{Boundedness with respect to the reference policy.}
    The Gibbs policy $\pi_R(a\mid x)$ satisfies
$e^{-\eta}
\le
\frac{\pi_{R}(a\mid x)}
     {\pi_{\text{ref}}(a\mid x)}
\le
e^{\eta}$,
for all $(x,a)\in\mathcal X\times\mathcal A$.
\item \textbf{Pairwise boundedness.}
For any two Gibbs policies $\pi_{R_1}$ and $\pi_{R_2}$ induced by reward functions
$R_1,R_2:\mathcal X\times\mathcal A\to[0,1]$,
$e^{-2\eta}
\le
\frac{\pi_{R_1}(a\mid x)}
     {\pi_{R_2}(a\mid x)}
\le
e^{2\eta}$,
for all $(x,a)\in\mathcal X\times\mathcal A$.
\item \textbf{Bounds on the normalizing constant.}
The normalizing constant
$Z_R(x)$
satisfies
$1
\le
Z_R(x)
\le
e^\eta$,
for all $x\in\mathcal X$.
\end{itemize}

\section{Proofs for Section \ref{sectionbandit}}

\subsection{Proof of Theorem \ref{theorem1}}

The proof of Theorem~\ref{theorem1} relies on two key lemmas. Lemma~\ref{lemma3} establishes the instantaneous KL-regularized regret decomposition used in Step~1 of the proof sketch, while Lemma~\ref{uniform convergence} provides the uniform convergence guarantee of the LS estimator required in Step~2. Combining these two ingredients yields the desired result.

\begin{lemma}[Instantaneous Regret Decomposition]
\label{lemma3}
For all $t\in[T]$, the instantaneous regret in the RF setting satisfies
\[
J_{\rm RF}(\pi^\star_{\operatorname{RF}})-J_{\rm RF}(\pi_t)
\le
\eta
\mathbb E_{x\sim d,\;a\sim \pi_t'}
\!\left[
\bigl(R^\star(x,a)-\widehat R_t(x,a)\bigr)^2
\right],
\]
where $\pi_t'$ is the Gibbs policy induced by
$R_t'=\gamma_t\widehat R_t+(1-\gamma_t)R^\star$
for some $\gamma_t\in[0,1]$.
\end{lemma}

\begin{proof}[Proof of Theorem~\ref{theorem1}]
Define $S_t
:=
\mathbb E_{x\sim d,\;a\sim\pi_t'}
\!\left[
\left(
\widehat R_t(x,a)-R^\star(x,a)
\right)^2
\right]$.
By Lemma~\ref{lemma3}, the cumulative regret admits the decomposition
\begin{equation}
\label{regretdecomposition}
\operatorname{Reg}_{\rm RF}(T)
=
\sum_{t=1}^{T}
\Bigl(
J_{\rm RF}(\pi^\star_{\text{RF}})-J_{\rm RF}(\pi_t)
\Bigr)
\le
\eta\sum_{t=1}^{T}S_t.
\end{equation}

We next derive an upper bound on $S_t$. By
Lemma~\ref{uniform convergence}, on an event of probability at least
$1-\delta$, the following holds simultaneously for all $t\ge2$:
\begin{equation}
\label{past-distribution-bound}
\sum_{i=1}^{t-1}
\mathbb E_{x\sim d,\;a\sim\pi_i}
\!\left[
\left(
R^\star(x,a)-\widehat R_t(x,a)
\right)^2
\right]
\le
72\log\frac{2N_{\mathcal R}t^3}{\delta},
\end{equation}
where $\pi_i$ denotes the policy used to generate the $i$-th sample.
It remains to transfer the above bound from the data-generating policies
$\{\pi_i\}_{i=1}^{t-1}$ to the intermediate Gibbs policy $\pi_t'$
appearing in Lemma~\ref{lemma3}. Fix $t\ge2$ and any $i<t$.
Since $\pi_i$ and $\pi_t'$ are Gibbs policies induced by $[0,1]$-valued reward
functions. Therefore
\[
S_t
=
\mathbb E_{x\sim d,\;a\sim\pi_t'}
\!\left[
\left(
R^\star(x,a)-\widehat R_t(x,a)
\right)^2
\right]
\le
e^{2\eta}
\mathbb E_{x\sim d,\;a\sim\pi_i}
\!\left[
\left(
R^\star(x,a)-\widehat R_t(x,a)
\right)^2
\right].
\]
Since the above inequality holds for every
$i=1,\ldots,t-1$, averaging over $i$ yields
\[
S_t
\le
\frac{e^{2\eta}}{t-1}
\sum_{i=1}^{t-1}
\mathbb E_{x\sim d,\;a\sim\pi_i}
\!\left[
\left(
R^\star(x,a)-\widehat R_t(x,a)
\right)^2
\right].
\]
Combining this inequality with Eq.~\eqref{past-distribution-bound} gives
\begin{equation}
\label{Sbound2}
S_t
\le
\frac{72e^{2\eta}}{t-1}
\log\frac{2N_{\mathcal R}t^3}{\delta},
\qquad
\forall t\ge2.
\end{equation}

For $t=1$, since both $\widehat R_1$ and $R^\star$ are
$[0,1]$-valued, we trivially have $S_1\le1$.
Therefore, using Eq.~\eqref{Sbound2},
\[
\sum_{t=1}^{T}S_t
\le
1
+
72e^{2\eta}
\log\frac{2N_{\mathcal R}T^3}{\delta}
\sum_{t=2}^{T}\frac{1}{t-1}.
\]
Applying the harmonic-series bound
$\sum_{t=2}^{T}\frac1{t-1}
\le
1+\log T$,
we obtain
\[
\sum_{t=1}^{T}S_t
\le
1
+
72e^{2\eta}
(1+\log T)
\log\frac{2N_{\mathcal R}T^3}{\delta}.
\]
Substituting the above estimate into Eq.~\eqref{regretdecomposition} yields
\[
\operatorname{Reg}_{\rm RF}(T)
\le
\eta
+
72\eta e^{2\eta}
(1+\log T)
\log\frac{2N_{\mathcal R}T^3}{\delta}.
\]
This completes the proof.
\end{proof}

\subsection{Proof of Lemma~\ref{lemma3}}

\begin{proof}[Proof of Lemma~\ref{lemma3}]
For any reward function
\(R:\mathcal X\times\mathcal A\to[0,1]\), define
\(\Delta_R(x,a):=R(x,a)-R^\star(x,a)\).
By the definition of the KL-regularized value,
\[
\begin{aligned}
J_{\rm RF}(\pi^\star_{\text{RF}})-J_{\rm RF}(\pi_R)
&=
\mathbb E_{x\sim d,\;a\sim\pi^\star}
\left[
R^\star(x,a)
-\eta^{-1}
\log\frac{\pi^\star_{\text{RF}}(a\mid x)}
{\pi_{\rm ref}(a\mid x)}
\right]
-
\mathbb E_{x\sim d,\;a\sim\pi_R}
\left[
R^\star(x,a)
-\eta^{-1}
\log\frac{\pi_R(a\mid x)}
{\pi_{\rm ref}(a\mid x)}
\right].
\end{aligned}
\]
Since \(\pi^\star_{\text{RF}}=\pi_{R^\star}\), we have
\(\log \frac{\pi^\star_{\text{RF}}(a\mid x)}{\pi_{\rm ref}(a\mid x)}
=
\eta R^\star(x,a)-\log Z_{R^\star}(x)\) and
\(\log \frac{\pi_R(a\mid x)}{\pi_{\rm ref}(a\mid x)}
=
\eta R(x,a)-\log Z_R(x)\).
Substituting these identities gives
\[
\begin{aligned}
J_{\rm RF}(\pi^\star_{\text{RF}})-J_{\rm RF}(\pi_R)
&=
\mathbb E_{x\sim d,\;a\sim\pi^\star_{\text{RF}}}
\left[
R^\star(x,a)
-\eta^{-1}
\bigl(
\eta R^\star(x,a)-\log Z_{R^\star}(x)
\bigr)
\right]
\\
&\quad-
\mathbb E_{x\sim d,\;a\sim\pi_R}
\left[
R^\star(x,a)
-\eta^{-1}
\bigl(
\eta R(x,a)-\log Z_R(x)
\bigr)
\right]
\\
&=
\eta^{-1}
\mathbb E_{x\sim d}
\left[
\log Z_{R^\star}(x)-\log Z_R(x)
\right]
+
\mathbb E_{x\sim d,\;a\sim\pi_R}
\left[
R(x,a)-R^\star(x,a)
\right].
\end{aligned}
\]
Define
\(\mathcal J_R(x)
:=
\log Z_R(x)
-
\eta
\sum_{a'\in\mathcal A}
\pi_R(a'\mid x)\Delta_R(x,a')\).
Since \(\Delta_{R^\star}(x,a)=0\) for all \((x,a)\), we have
\(\mathcal J_{R^\star}(x)=\log Z_{R^\star}(x)\), and therefore
\[
J_{\rm RF}(\pi^\star_{\text{RF}})-J_{\rm RF}(\pi_R)
=
\eta^{-1}
\mathbb E_{x\sim d}
\left[
\mathcal J_{R^\star}(x)-\mathcal J_R(x)
\right].
\]

We next view \(\mathcal J_R(x)\) as a function of
\(\Delta_R(x,a)\). Since
\[
Z_R(x)
=
\sum_{a'\in\mathcal A}
\pi_{\rm ref}(a'\mid x)
\exp\!\left(
\eta
\bigl(
R^\star(x,a')+\Delta_R(x,a')
\bigr)
\right),
\]
we have
\[
\begin{aligned}
\frac{\partial \log Z_R(x)}
{\partial \Delta_R(x,a)}
&=
\frac{1}{Z_R(x)}
\frac{\partial Z_R(x)}
{\partial \Delta_R(x,a)}
\\
&=
\frac{
\eta\pi_{\rm ref}(a\mid x)
\exp\!\left(
\eta
\bigl(
R^\star(x,a)+\Delta_R(x,a)
\bigr)
\right)
}{
Z_R(x)
}
\\
&=
\eta\pi_R(a\mid x).
\end{aligned}
\]
Moreover,
\[
\begin{aligned}
\frac{\partial \pi_R(a'\mid x)}
{\partial \Delta_R(x,a)}
&=
\frac{\partial}{\partial \Delta_R(x,a)}
\left(
\frac{
\pi_{\rm ref}(a'\mid x)
\exp\!\bigl(\eta R(x,a')\bigr)
}{
Z_R(x)
}
\right)
\\
&=
\frac{
\eta\mathbf 1\{a'=a\}
\pi_{\rm ref}(a'\mid x)
\exp\!\bigl(\eta R(x,a')\bigr)
}{
Z_R(x)
}
-
\frac{
\pi_{\rm ref}(a'\mid x)
\exp\!\bigl(\eta R(x,a')\bigr)
}{
Z_R(x)^2
}
\frac{\partial Z_R(x)}
{\partial \Delta_R(x,a)}
\\
&=
\eta\pi_R(a'\mid x)\mathbf 1\{a'=a\}
-
\eta\pi_R(a'\mid x)\pi_R(a\mid x)
\\
&=
\eta\pi_R(a'\mid x)
\bigl(
\mathbf 1\{a'=a\}-\pi_R(a\mid x)
\bigr).
\end{aligned}
\]
Therefore,
\[
\begin{aligned}
\frac{\partial \mathcal J_R(x)}
{\partial \Delta_R(x,a)}
&=
\eta\pi_R(a\mid x)
-\eta\pi_R(a\mid x)
-\eta
\sum_{a'\in\mathcal A}
\Delta_R(x,a')
\frac{\partial \pi_R(a'\mid x)}
{\partial \Delta_R(x,a)}
\\
&=
-\eta^2
\sum_{a'\in\mathcal A}
\pi_R(a'\mid x)
\bigl(
\mathbf 1\{a'=a\}
-\pi_R(a\mid x)
\bigr)
\Delta_R(x,a')
\\
&=
-\eta^2
\pi_R(a\mid x)\Delta_R(x,a)
+
\eta^2
\pi_R(a\mid x)
\sum_{a'\in\mathcal A}
\pi_R(a'\mid x)\Delta_R(x,a').
\end{aligned}
\]

To apply the mean value theorem, define the interpolation
\(R_\gamma:=\gamma R+(1-\gamma)R^\star\) for
\(\gamma\in[0,1]\), and let
\(g(\gamma):=
\mathbb E_{x\sim d}
[\mathcal J_{R_\gamma}(x)]\).
Since \(R_0=R^\star\) and \(R_1=R\), by the one-dimensional mean value theorem there exists some
\(\gamma\in[0,1]\) such that, with
\(R':=R_\gamma=\gamma R+(1-\gamma)R^\star\),
\[
\mathbb E_{x\sim d}
\left[
\mathcal J_{R^\star}(x)-\mathcal J_R(x)
\right]
=
-g'(\gamma).
\]
Furthermore,
\[
\Delta_{R_\gamma}(x,a)
=
\gamma\Delta_R(x,a),
\qquad
\frac{d}{d\gamma}\Delta_{R_\gamma}(x,a)
=
\Delta_R(x,a).
\]
In particular, since \(R'=R_\gamma\), we have
\(\Delta_{R'}(x,a)=\gamma\Delta_R(x,a)\). Therefore,
\[
\begin{aligned}
&\eta^{-1}
\mathbb E_{x\sim d}
\left[
\mathcal J_{R^\star}(x)-\mathcal J_R(x)
\right]
\\
&=
-\eta^{-1}
\mathbb E_{x\sim d}
\left[
\sum_{a\in\mathcal A}
\frac{
\partial \mathcal J_{R'}(x)
}{
\partial \Delta_{R'}(x,a)
}
\Delta_R(x,a)
\right]
\\
&=
\eta\gamma
\mathbb E_{x\sim d}
\left[
\sum_{a\in\mathcal A}
\pi_{R'}(a\mid x)
\Delta_R(x,a)^2
\right]
\\
&\quad-
\eta\gamma
\mathbb E_{x\sim d}
\left[
\sum_{a_1,a_2\in\mathcal A}
\pi_{R'}(a_1\mid x)
\pi_{R'}(a_2\mid x)
\Delta_R(x,a_1)
\Delta_R(x,a_2)
\right]
\\
&=
\eta\gamma
\mathbb E_{x\sim d}
\left[
\operatorname{Var}_{a\sim\pi_{R'}(\cdot\mid x)}
\bigl(\Delta_R(x,a)\bigr)
\right].
\end{aligned}
\]
Since \(\gamma\in[0,1]\), it follows that
\[
\begin{aligned}
&\eta^{-1}
\mathbb E_{x\sim d}
\left[
\mathcal J_{R^\star}(x)-\mathcal J_R(x)
\right]
\\
&\le
\eta\gamma
\mathbb E_{x\sim d,\;a\sim\pi_{R'}}
\left[
\Delta_R(x,a)^2
\right]
\\
&\le
\eta
\mathbb E_{x\sim d,\;a\sim\pi_{R'}}
\left[
\Delta_R(x,a)^2
\right].
\end{aligned}
\]
The term being subtracted is nonnegative, since
\[
\begin{aligned}
&\sum_{a_1,a_2\in\mathcal A}
\pi_{R'}(a_1\mid x)
\pi_{R'}(a_2\mid x)
\Delta_R(x,a_1)
\Delta_R(x,a_2)
=
\left(
\sum_{a\in\mathcal A}
\pi_{R'}(a\mid x)
\Delta_R(x,a)
\right)^2
=
\left(
\mathbb E_{a\sim\pi_{R'}}
[\Delta_R(x,a)]
\right)^2
\ge0.
\end{aligned}
\]
Hence,
\[
\eta^{-1}
\mathbb E_{x\sim d}
\left[
\mathcal J_{R^\star}(x)-\mathcal J_R(x)
\right]
\le
\eta
\mathbb E_{x\sim d,\;a\sim\pi_{R'}}
\left[
\Delta_R(x,a)^2
\right].
\]
Combining this inequality with the previous expression for
\(J_{\rm RF}(\pi^\star_{\text{RF}})-J_{\rm RF}(\pi_R)\) yields
\[
J_{\rm RF}(\pi^\star_{\text{RF}})-J_{\rm RF}(\pi_R)
\le
\eta
\mathbb E_{x\sim d,\;a\sim\pi_{R'}}
\left[
\left(
R^\star(x,a)-R(x,a)
\right)^2
\right].
\]

Finally, taking \(R=\widehat R_t\), there exists
\(\gamma_t\in[0,1]\) such that
\(R_t'
=
\gamma_t\widehat R_t+(1-\gamma_t)R^\star\).
Since
\(\pi_{\widehat R_t}=\pi_t\) and
\(\pi_{R_t'}=\pi_t'\), we obtain
\[
J_{\rm RF}(\pi^\star_{\text{RF}})-J_{\rm RF}(\pi_t)
\le
\eta
\mathbb E_{x\sim d,\;a\sim\pi_t'}
\left[
\left(
R^\star(x,a)-\widehat R_t(x,a)
\right)^2
\right].
\]
This completes the proof.
\end{proof}

\subsection{Proof of Lemma \ref{uniform convergence}}

\begin{proof}[Proof of Lemma \ref{uniform convergence}]
Fix any $\delta\in(0,1)$. For each $t\ge2$, define
$\delta_t:=\frac{\delta}{2t^3}$.
We first establish a uniform bound over $R\in\mathcal R$ for each
fixed $t$, and then apply a union bound over $t$.

Fix any $t\ge2$ and $R\in\mathcal R$. For each $i<t$, define
$Y_{R,i}
:=
\left(R(x_i,a_i)-r_i\right)^2
-
\left(R^\star(x_i,a_i)-r_i\right)^2$.
Further define
$M_{R,i}
:=
\mathbb E[Y_{R,i}\mid\mathcal H^{\text{RF}}_{i-1}]
-
Y_{R,i}$.
Then $\{M_{R,i}\}_{i\ge1}$ is a martingale difference sequence with
respect to $\{\mathcal H_i^{\text{RF}}\}_{i\ge0}$. Moreover, $\operatorname{Var}
\left[
M_{R,i}\mid\mathcal H^{\text{RF}}_{i-1}
\right]
=
\operatorname{Var}
\left[
Y_{R,i}\mid\mathcal H^{\text{RF}}_{i-1}
\right]$.
Since $R$, $R^\star$, and $r_i$ are all $[0,1]$-valued,
$|Y_{R,i}|\le1$, and hence $|M_{R,i}|\le2$.

Applying Freedman's inequality (Lemma~\ref{freedman}) with confidence
level $\delta_t/N_{\mathcal R}$ yields that, with probability at least
$1-\log_2(t-1)\delta_t/N_{\mathcal R}$,
\[
\begin{aligned}
\sum_{i=1}^{t-1}M_{R,i}
&\le
4\sqrt{
\sum_{i=1}^{t-1}
\operatorname{Var}
\left[
Y_{R,i}\mid\mathcal H^{\text{RF}}_{i-1}
\right]
\log\frac{N_{\mathcal R}}{\delta_t}
}
+
4\log\frac{N_{\mathcal R}}{\delta_t}.
\end{aligned}
\]
Substituting the definition of $M_{R,i}$ gives
\[
\begin{aligned}
\sum_{i=1}^{t-1}
\mathbb E
\left[
Y_{R,i}\mid\mathcal H^{\text{RF}}_{i-1}
\right]
-
\sum_{i=1}^{t-1}Y_{R,i}
&\le
4\sqrt{
\sum_{i=1}^{t-1}
\operatorname{Var}
\left[
Y_{R,i}\mid\mathcal H^{\text{RF}}_{i-1}
\right]
\log\frac{N_{\mathcal R}}{\delta_t}
}
+
4\log\frac{N_{\mathcal R}}{\delta_t}.
\end{aligned}
\]
Taking a union bound over all $R\in\mathcal R$, with probability at
least $1-\log_2(t-1)\delta_t$, the above inequality holds
simultaneously for all $R\in\mathcal R$.

By Lemma~\ref{lemma6},
$\operatorname{Var}
\left[
Y_{R,i}\mid\mathcal H^{\text{RF}}_{i-1}
\right]
\le
4\mathbb E
\left[
Y_{R,i}\mid\mathcal H^{\text{RF}}_{i-1}
\right]$.
Therefore,
\[
\begin{aligned}
\sum_{i=1}^{t-1}
\mathbb E
\left[
Y_{R,i}\mid\mathcal H^{\text{RF}}_{i-1}
\right]
&\le
8\sqrt{
\sum_{i=1}^{t-1}
\mathbb E
\left[
Y_{R,i}\mid\mathcal H^{\text{RF}}_{i-1}
\right]
\log\frac{N_{\mathcal R}}{\delta_t}
}+
4\log\frac{N_{\mathcal R}}{\delta_t}
+
\sum_{i=1}^{t-1}Y_{R,i},
\qquad
\forall R\in\mathcal R.
\end{aligned}
\]
Rearranging the above inequality and completing the square gives
\[
\left(
\sqrt{
\sum_{i=1}^{t-1}
\mathbb E
\left[
Y_{R,i}\mid\mathcal H^{\text{RF}}_{i-1}
\right]
}
-
4\sqrt{
\log\frac{N_{\mathcal R}}{\delta_t}
}
\right)^2
\le
20\log\frac{N_{\mathcal R}}{\delta_t}
+
\sum_{i=1}^{t-1}Y_{R,i}.
\]
Consequently,
\[
\begin{aligned}
\sqrt{
\sum_{i=1}^{t-1}
\mathbb E
\left[
Y_{R,i}\mid\mathcal H^{\text{RF}}_{i-1}
\right]
}
&\le
4\sqrt{\log\frac{N_{\mathcal R}}{\delta_t}}
+
\sqrt{
20\log\frac{N_{\mathcal R}}{\delta_t}
+
\sum_{i=1}^{t-1}Y_{R,i}
}.
\end{aligned}
\]
Squaring both sides and using $(a+b)^2\le2a^2+2b^2$, we obtain
\[
\sum_{i=1}^{t-1}
\mathbb E
\left[
Y_{R,i}\mid\mathcal H^{\text{RF}}_{i-1}
\right]
\le
72\log\frac{N_{\mathcal R}}{\delta_t}
+
2\sum_{i=1}^{t-1}Y_{R,i},
\qquad
\forall R\in\mathcal R.
\]

Conditioned on $\mathcal H^{\text{RF}}_{i-1}$, the context satisfies
$x_i\sim d$ and the action is sampled according to
$a_i\sim\pi_i(\cdot\mid x_i)$. Hence, by Lemma~\ref{lemma6}, $\mathbb E
\left[
Y_{R,i}\mid\mathcal H^{\text{RF}}_{i-1}
\right]
=
\mathbb E_{x\sim d,\;a\sim\pi_i}
\left[
\left(
R^\star(x,a)-R(x,a)
\right)^2
\right]$.
It follows that, with probability at least
$1-\log_2(t-1)\delta_t$,
\[
\begin{aligned}
\sum_{i=1}^{t-1}
\mathbb E_{x\sim d,\;a\sim\pi_i}
\left[
\left(
R^\star(x,a)-R(x,a)
\right)^2
\right]
&\le
72\log\frac{N_{\mathcal R}}{\delta_t}
+
2\sum_{i=1}^{t-1}Y_{R,i},
\end{aligned}
\]
simultaneously for all $R\in\mathcal R$.

We now take a union bound over $t\ge2$. Since $\sum_{t=2}^{\infty}
\delta_t\log_2(t-1)
\le
\sum_{t=2}^{\infty}
\frac{\delta}{2t^2}
\le
\delta$,
with probability at least $1-\delta$, the following holds
simultaneously for all $t\ge2$ and all $R\in\mathcal R$:
\[
\begin{aligned}
\sum_{i=1}^{t-1}
\mathbb E_{x\sim d,\;a\sim\pi_i}
\left[
\left(
R^\star(x,a)-R(x,a)
\right)^2
\right]
&\le
72\log\frac{2N_{\mathcal R}t^3}{\delta}
+
2\sum_{i=1}^{t-1}
\left[
\left(R(x_i,a_i)-r_i\right)^2
-
\left(R^\star(x_i,a_i)-r_i\right)^2
\right].
\end{aligned}
\]

For each $t\ge2$, take $R=\widehat R_t$, where $\widehat R_t$ is the
LS estimator over $\mathcal R$ constructed from
$\{(x_i,a_i,r_i)\}_{i=1}^{t-1}$.
Since $R^\star\in\mathcal R$ by Assumption~\ref{assumption1}, the
optimality of $\widehat R_t$ implies
\[
\sum_{i=1}^{t-1}
\left[
\left(
\widehat R_t(x_i,a_i)-r_i
\right)^2
-
\left(
R^\star(x_i,a_i)-r_i
\right)^2
\right]
\le0.
\]
Substituting this inequality into the previous display gives
\[
\sum_{i=1}^{t-1}
\mathbb E_{x\sim d,\;a\sim\pi_i}
\left[
\left(
R^\star(x,a)-\widehat R_t(x,a)
\right)^2
\right]
\le
72\log\frac{2N_{\mathcal R}t^3}{\delta},
\]
simultaneously for all $t\ge2$.
In particular, for all $t=2,\ldots,T$,
\[
\sum_{i=1}^{t-1}
\mathbb E_{x\sim d,\;a\sim\pi_i}
\left[
\left(
R^\star(x,a)-\widehat R_t(x,a)
\right)^2
\right]
\le
72\log\frac{2N_{\mathcal R}T^3}{\delta}.
\]
This completes the proof.
\end{proof}

\subsection{Proof of Corollary \ref{corollary1}}

\begin{proof}[Proof of Corollary~\ref{corollary1}]
By Lemma~\ref{lemmac1}, applied with confidence level
$\delta/2$, with probability at least $1-\delta/2$,
simultaneously for all $t\in[T]$,
\[
\sum_{i=1}^{t-1}
\left(
\widehat R_t(x_i,a_i)-R^\star(x_i,a_i)
\right)^2
\le
8\log\frac{2N_{\mathcal R}T}{\delta}.
\]
Condition on this event. For any $(x,a)$, since
$\widehat R_t,R^\star\in\mathcal R$, the definition of the
uncertainty measure gives
\[
\begin{aligned}
\left|
R^\star(x,a)-\widehat R_t(x,a)
\right|
&\le
U_{\rm RF}(\lambda,x,a,\mathcal R;D^{\text{RF}}_{t-1})
\sqrt{
\lambda+
\sum_{i=1}^{t-1}
\left(
R^\star(x_i,a_i)-\widehat R_t(x_i,a_i)
\right)^2
}.
\end{aligned}
\]
Using
$\lambda\le8\log\frac{2N_{\mathcal R}T}{\delta}$, we obtain
\[
\left(
R^\star(x,a)-\widehat R_t(x,a)
\right)^2
\le
16\log\frac{2N_{\mathcal R}T}{\delta}
\,U_{\rm RF}^2(\lambda,x,a,\mathcal R;\mathcal D^{\text{RF}}_{t-1}).
\]
Since both reward functions take values in $[0,1]$, the
left-hand side is also at most $1$. Hence,
\[
\left(
R^\star(x,a)-\widehat R_t(x,a)
\right)^2
\le
16\log\frac{2N_{\mathcal R}T}{\delta}
\min\left\{
1,
U_{\rm RF}^2(\lambda,x,a,\mathcal R;\mathcal D^{\text{RF}}_{t-1})
\right\}.
\]

Recall that $S_t
=
\mathbb E_{x\sim d,\;a\sim\pi_t'}
\left[
\left(
R^\star(x,a)-\widehat R_t(x,a)
\right)^2
\right]$.
Therefore,
\[
\begin{aligned}
S_t
&\le
16\log\frac{2N_{\mathcal R}T}{\delta}
\,
\mathbb E_{x\sim d,\;a\sim\pi_t'}
\left[
\min\left\{
1,
U_{\rm RF}^2(\lambda,x,a,\mathcal R;\mathcal D^{\text{RF}}_{t-1})
\right\}
\right].
\end{aligned}
\]
Since $\pi_t'$ and $\pi_t$ are Gibbs policies induced by
$[0,1]$-valued reward functions, $\frac{\pi_t'(a\mid x)}{\pi_t(a\mid x)}
\le e^{2\eta},\ \forall (x,a)$, we have
\[
\begin{aligned}
S_t
&\le
16e^{2\eta}
\log\frac{2N_{\mathcal R}T}{\delta}
\mathbb E_{x\sim d,\;a\sim\pi_t}
\left[
\min\left\{
1,
U_{\rm RF}^2(\lambda,x,a,\mathcal R;\mathcal D^{\text{RF}}_{t-1})
\right\}
\right].
\end{aligned}
\]

Now let $X_t
:=
\min\left\{
1,
U_{\rm RF}^2(\lambda,x_t,a_t,\mathcal R;\mathcal D^{\text{RF}}_{t-1})
\right\}$.
Conditioned on the history before round $t$,
$x_t\sim d$ and $a_t\sim\pi_t(\cdot\mid x_t)$, and hence
$\mathbb E[X_t\mid\mathcal H^{\text{RF}}_{t-1}]
=
\mathbb E_{x\sim d,\;a\sim\pi_t}
\left[
\min\left\{
1,
U_{\rm RF}^2(\lambda,x,a,\mathcal R;\mathcal D^{\text{RF}}_{t-1})
\right\}
\right]$.
Moreover, by the definition of the eluder dimension,
$\sum_{t=1}^T X_t
\le
d_{\rm RF}(\lambda,\mathcal R,T)$
for every realized sequence.

Since $X_t\in[0,1]$, using
$e^{-x}\le1-(1-e^{-1})x$ for $x\in[0,1]$, we have
\[
\begin{aligned}
\mathbb E
\left[
e^{-X_t}\mid\mathcal H^{\text{RF}}_{t-1}
\right]
&\le
1-(1-e^{-1})
\mathbb E[X_t\mid\mathcal H^{\text{RF}}_{t-1}]
\\
&\le
\exp\left(
-(1-e^{-1})
\mathbb E[X_t\mid\mathcal H^{\text{RF}}_{t-1}]
\right).
\end{aligned}
\]
Therefore,
\[
\exp\left(
(1-e^{-1})
\sum_{t=1}^T
\mathbb E[X_t\mid\mathcal H^{\text{RF}}_{t-1}]
-
\sum_{t=1}^T X_t
\right)
\]
has expectation at most one. By Markov's inequality, with
probability at least $1-\delta/2$,
\[
(1-e^{-1})
\sum_{t=1}^T
\mathbb E[X_t\mid\mathcal H^{\text{RF}}_{t-1}]
-
\sum_{t=1}^T X_t
\le
\log\frac{2}{\delta}.
\]
Combining this with the eluder-dimension bound yields
\[
\sum_{t=1}^T
\mathbb E[X_t\mid\mathcal H^{\text{RF}}_{t-1}]
\le
\frac{
d_{\rm RF}(\lambda,\mathcal R,T)+\log\frac{2}{\delta}
}{
1-e^{-1}
}.
\]

Taking a union bound over the two high-probability events,
with probability at least $1-\delta$,
\[
\sum_{t=1}^T S_t
\le
\frac{16e^{2\eta}}{1-e^{-1}}
\left(
d_{\rm RF}(\lambda,\mathcal R,T)+\log\frac{2}{\delta}
\right)
\log\frac{2N_{\mathcal R}T}{\delta}.
\]
Finally, by Lemma~\ref{lemma3}, $\operatorname{Reg}_{\rm RF}(T)
\le
\eta\sum_{t=1}^T S_t$,
and therefore
\[
\operatorname{Reg}_{\rm RF}(T)
=
O\!\left(
\eta e^{2\eta}
\left(
d_{\rm RF}(\lambda,\mathcal R,T)+\log\frac1\delta
\right)
\log\frac{N_{\mathcal R}T}{\delta}
\right).
\]
\end{proof}

\subsection{Proof of Corollary~\ref{corollary2}}

We next consider \texttt{K-UCB}~\citep{zhao2025logarithmic},
which augments the least-squares reward estimate with an
uncertainty-based exploration bonus and selects the Gibbs policy
induced by the resulting optimistic reward function. Specifically,
at round $t$, it uses
\[
b_t(x,a)
=
\min\left\{
1,\;
U_{\rm RF}(\lambda,x,a,\mathcal R;\mathcal D^\text{RF}_{t-1})
\sqrt{16\log\frac{2N_{\mathcal R}T}{\delta}}
\right\},
\]
and selects
\[
\pi_t(a\mid x)
\propto
\pi_{\rm ref}(a\mid x)
\exp\left(
\eta\bigl(\widehat R_t(x,a)+b_t(x,a)\bigr)
\right).
\]

\begin{corollary}[$\sqrt{T}$ Regret Bound for \texttt{K-UCB}]
\label{corollary2}
Fix any $\lambda>0$ satisfying
$\lambda
\le
8\log\frac{2N_{\mathcal R}T}{\delta}$.
Then, under Assumption~\ref{assumption1}, for any
$\delta\in(0,1)$, with probability at least $1-\delta$, the regret
of \texttt{K-UCB} satisfies
\[
\operatorname{Reg}_{\rm RF}(T)
=
O\!\left(
\sqrt{
T
\left(
d_{\rm RF}(\lambda,\mathcal R,T)
+\log\frac{1}{\delta}
\right)
\log\frac{N_{\mathcal R}T}{\delta}
}
\right).
\]
\end{corollary}

\begin{proof}[Proof of Corollary~\ref{corollary2}]
On the same high-probability confidence event used in the proof of
Corollary~\ref{corollary1}, the construction of the bonus ensures
\[
\left|
\widehat R_t(x,a)-R^\star(x,a)
\right|
\le
b_t(x,a),
\qquad \forall (x,a),\ t\in[T].
\]
By optimism and the optimality of $\pi_t$ under the optimistic reward
function $\widehat R_t+b_t$, we have
\begin{align*}
\operatorname{Reg}_{\rm RF}(T)
&\le
\sum_{t=1}^T
\mathbb E_{x\sim d,\;a\sim\pi_t}
\left[
\widehat R_t(x,a)+b_t(x,a)-R^\star(x,a)
\right]
\\
&\le
2\sum_{t=1}^T
\mathbb E_{x\sim d,\;a\sim\pi_t}
\left[
b_t(x,a)
\right]
\\
&\le
8\sqrt{\log\frac{2N_{\mathcal R}T}{\delta}}
\sum_{t=1}^T
\mathbb E_{x\sim d,\;a\sim\pi_t}
\left[
\min\left\{
1,
U_{\rm RF}(\lambda,x,a,\mathcal R;\mathcal D^{\text{RF}}_{t-1})
\right\}
\right].
\end{align*}
By Cauchy--Schwarz,
\begin{align*}
\operatorname{Reg}_{\rm RF}(T)
&\le
8\sqrt{
T\log\frac{2N_{\mathcal R}T}{\delta}
\sum_{t=1}^T
\mathbb E_{x\sim d,\;a\sim\pi_t}
\left[
\min\left\{
1,
U_{\rm RF}^2(\lambda,x,a,\mathcal R;\mathcal D^{\text{RF}}_{t-1})
\right\}
\right]
}.
\end{align*}
The predictable-to-realized uncertainty bound established in the
proof of Corollary~\ref{corollary1} gives, with high probability,
\[
\sum_{t=1}^T
\mathbb E_{x\sim d,\;a\sim\pi_t}
\left[
\min\left\{
1,
U_{\rm RF}^2(\lambda,x,a,\mathcal R;\mathcal D^{\text{RF}}_{t-1})
\right\}
\right]
=
O\!\left(
d_{\rm RF}(\lambda,\mathcal R,T)
+\log\frac{1}{\delta}
\right).
\]
Substituting this bound into the previous inequality yields
\[
\operatorname{Reg}_{\rm RF}(T)
=
O\!\left(
\sqrt{
T
\left(
d_{\rm RF}(\lambda,\mathcal R,T)
+\log\frac{1}{\delta}
\right)
\log\frac{N_{\mathcal R}T}{\delta}
}
\right),
\]
which completes the proof.
\end{proof}

\section{Proofs for Section \ref{sectionpreference}}

\subsection{General Preference Model}

Recall that $\widehat{\pi}_t^1$ denotes the symmetric
Nash-equilibrium policy induced by $\widehat P_t$, and define its true
regularized best response as
\[
\widetilde{\pi}_t^2
:=
\arg\min_{\pi^2}
J_{\operatorname{GP}}
\left(
    \widehat{\pi}_t^1,\pi^2
\right).
\]

\subsubsection{Proof of Theorem \ref{theorem2}}

The proof of Theorem~\ref{theorem2} relies on the following two key lemmas.
Lemma~\ref{lem:gp_regret_decomposition} establishes the instantaneous KL-regularized regret decomposition under the GP model, while Lemma~\ref{uniformconvergnce_preference} provides the uniform convergence guarantee for the MLE estimator under the GP model.

\begin{lemma}[Instantaneous Regret Decomposition for the General Preference Model]
\label{lem:gp_regret_decomposition} 
For all $t\in[T]$, the instantaneous regret under the GP setting satisfies
\[
J_{\mathrm{GP}}^\star
-
J_{\mathrm{GP}}
(\widehat{\pi}_t^1,\widetilde{\pi}_t^2)
\le
2\eta e^\eta
\mathbb E_{\substack{
x\sim d,\,
a^1\sim\widehat{\pi}_t^1,\,
a^2\sim\pi_{\mathrm{ref}}
}}
\left[
\left(
P^\star(x,a^1,a^2)
-
\widehat P_t(x,a^1,a^2)
\right)^2
\right].
\]
\end{lemma}

\begin{lemma}[Uniform Prediction Error Bound, General Preference Model]
\label{uniformconvergnce_preference} Suppose Assumption~\ref{assumption2} holds.
Let $\widehat P_t$ be the MLE estimator constructed from the
samples generated by \texttt{ORLHF-GS} under the GP model, where,
conditionally on $\mathcal H_{i-1}^{\rm GP}$ and $x_i$,
$a_i^1$ and $a_i^2$ are sampled independently according to
$\widehat\pi_i^1(\cdot\mid x_i)$ and
$\pi_{\rm ref}(\cdot\mid x_i)$, respectively.
Then for any such policy sequence
$\{\widehat\pi_i^1\}_{i\ge1}$ and any $\delta\in(0,1)$,
with probability at least $1-\delta$, the following holds
simultaneously for all $t = 2,\dots,T$:
\[
\sum_{i=1}^{t-1}
\mathbb E_{x\sim d,\;a^1\sim\widehat\pi_i^1,\;a^2\sim\pi_{\operatorname{ref}}}
\left[
\left(
P^\star(x,a^1,a^2)-\widehat P_t(x,a^1,a^2)
\right)^2
\right]
\le
6\log\frac{2N_{\mathcal P}T^3}{\delta}.
\]
\end{lemma}

\begin{proof}[Proof of Theorem~\ref{theorem2}]
For each $t\in[T]$, define
$S_t
:=
\mathbb E_{\substack{x\sim d,\,
a^1\sim\widehat\pi_t^1,\ 
a^2\sim\pi_{\rm ref}}}
\left[
\bigl(
P^\star(x,a^1,a^2)
-
\widehat P_t(x,a^1,a^2)
\bigr)^2
\right]$.
By Lemma~\ref{lem:gp_regret_decomposition},
\begin{equation}
\operatorname{Reg}_{\rm GP}(T)
=
O\!\left(
\eta e^{\eta}
\sum_{t=1}^T S_t
\right).
\label{eq:gp-reg-decomp}
\end{equation}

We establish two different bounds on
$\sum_{t=1}^T S_t$.

\paragraph{Dimension-independent bound.}
By Lemma~\ref{uniformconvergnce_preference}, with probability at least
$1-\delta$, simultaneously for all $t=2,\ldots,T$,
\[
\sum_{i=1}^{t-1}
\mathbb E_{\substack{x\sim d,\,
a^1\sim\widehat\pi_i^1,\
a^2\sim\pi_{\rm ref}}}
\left[
\bigl(
P^\star(x,a^1,a^2)
-
\widehat P_t(x,a^1,a^2)
\bigr)^2
\right]
\le
6\log\frac{2N_{\mathcal P}T^3}{\delta}.
\]
For any $i<t$, the Gibbs-policy likelihood-ratio bound gives $\frac{\widehat\pi_t^1(a\mid x)}
     {\widehat\pi_i^1(a\mid x)}
\le e^{2\eta}$.
Therefore,
\[
S_t
\le
\frac{6e^{2\eta}}{t-1}
\log\frac{2N_{\mathcal P}T^3}{\delta},
\qquad t\ge2.
\]
Since $S_1\le1$,
\[
\sum_{t=1}^T S_t
=
O\!\left(
e^{2\eta}
\log T
\log\frac{N_{\mathcal P}T}{\delta}
\right).
\]
Substituting into~\eqref{eq:gp-reg-decomp} yields
\[
\operatorname{Reg}_{\rm GP}(T)
=
O\!\left(
(\eta e^{3\eta})
\log T
\log\frac{N_{\mathcal P}T}{\delta}
\right).
\]

\paragraph{Eluder-dimension-dependent bound.}
By Lemma~\ref{lemma3wu}, with probability at least
$1-\delta/2$, simultaneously for all $t\in[T]$,
\[
\sum_{i=1}^{t-1}
\bigl(
\widehat P_t(x_i,a_i^1,a_i^2)
-
P^\star(x_i,a_i^1,a_i^2)
\bigr)^2
\le 2\log\frac{2N_{\mathcal P}T}{\delta}.
\]
Fix any $(x,a^1,a^2)$. Since
$\widehat P_t,P^\star\in\mathcal P$, the definition of
$U_{\rm GP}$ implies
\begin{align*}
&
\bigl(
\widehat P_t(x,a^1,a^2)
-
P^\star(x,a^1,a^2)
\bigr)^2
\le
U_{\rm GP}^2(
\lambda,x,a^1,a^2,\mathcal P;\mathcal D^\text{GP}_{t-1})
\left[
\lambda+
\sum_{i=1}^{t-1}
\bigl(
\widehat P_t(x_i,a_i^1,a_i^2)
-
P^\star(x_i,a_i^1,a_i^2)
\bigr)^2
\right].
\end{align*}
Hence, for $\lambda\le \log\frac{2N_{\mathcal P}T}{\delta}$, using
$|\widehat P_t-P^\star|\le1$,
\[
\bigl(
\widehat P_t(x,a^1,a^2)
-
P^\star(x,a^1,a^2)
\bigr)^2
\le
3\log\frac{2N_{\mathcal P}T}{\delta}
\min\!\left\{
1,
U_{\rm GP}^2(
\lambda,x,a^1,a^2,\mathcal P;\mathcal D^\text{GP}_{t-1})
\right\}.
\]
Consequently,
\[
S_t
\le
3\log\frac{2N_{\mathcal P}T}{\delta}\,
\mathbb E_{\substack{x\sim d,\,
a^1\sim\widehat\pi_t^1,\
a^2\sim\pi_{\rm ref}}}
\left[
\min\!\left\{
1,
U_{\rm GP}^2(
\lambda,x,a^1,a^2,\mathcal P;\mathcal D^\text{GP}_{t-1})
\right\}
\right].
\]

Applying the same predictable-to-realized concentration argument
as in the proof of Corollary~\ref{corollary1} with confidence level
$\delta/2$, together with the definition of
$d_{\rm GP}(\lambda,\mathcal P,T)$, and taking a union bound with the
preceding MLE confidence event, yields, with probability at least
$1-\delta$,
\[
\sum_{t=1}^T S_t
=
O\!\left(
\left(
d_{\rm GP}(\lambda,\mathcal P,T)
+\log\frac{1}{\delta}
\right)
\log\frac{N_{\mathcal P}T}{\delta}
\right).
\]
Substituting this bound into~\eqref{eq:gp-reg-decomp} gives
\[
\operatorname{Reg}_{\rm GP}(T)
=
O\!\left(
\eta e^{\eta}
\left(
d_{\rm GP}(\lambda,\mathcal P,T)
+\log\frac{1}{\delta}
\right)
\log\frac{N_{\mathcal P}T}{\delta}
\right).
\]
\end{proof}

\subsubsection{Proof of Lemma \ref{lem:gp_regret_decomposition}}
    
\begin{proof}[Proof of Lemma \ref{lem:gp_regret_decomposition}]
For brevity, denote the instantaneous regret on the left-hand side by $G_t
:=
J_{\mathrm{GP}}^\star
-
J_{\mathrm{GP}}
(\widehat{\pi}_t^1,\widetilde{\pi}_t^2)$.

By reciprocity, for any $P\in\mathcal P$ and any policy $\pi$,
\[
P(x,\pi,\pi)
=
\mathbb E_{a^1\sim\pi,\,a^2\sim\pi}
[P(x,a^1,a^2)]
=
\frac12.
\]
Indeed, exchanging $a^1$ and $a^2$ does not change their joint
distribution, while
$P(x,a^1,a^2)+P(x,a^2,a^1)=1$.
Moreover, when the two policies coincide, the two
KL-regularization terms in $J_{\mathrm{GP}}$ cancel. Consequently, $J_{\mathrm{GP}}(\pi,\pi)=\frac12$
for every $\pi$, and in particular $J_{\mathrm{GP}}^\star
=
J_{\mathrm{GP}}
(\widehat{\pi}_t^1,\widehat{\pi}_t^1)
=
\frac12$.
Therefore,
\begin{equation}
G_t
=
J_{\mathrm{GP}}
(\widehat{\pi}_t^1,\widehat{\pi}_t^1)
-
J_{\mathrm{GP}}
(\widehat{\pi}_t^1,\widetilde{\pi}_t^2).
\label{eq:gp-new-regret-gap}
\end{equation}

Since $\widehat{\pi}_t^1$ is the symmetric
Nash-equilibrium policy induced by $\widehat P_t$, it is also the
minimizing player's regularized best response to itself under
$\widehat P_t$. Hence,
\begin{align}
&
\mathbb E_{x\sim d}
\left[
\widehat P_t
(x,\widehat{\pi}_t^1,\widehat{\pi}_t^1)
+
\eta^{-1}
\operatorname{KL}
(\widehat{\pi}_t^1,\pi_{\mathrm{ref}}\mid x)
\right]\le
\mathbb E_{x\sim d}
\left[
\widehat P_t
(x,\widehat{\pi}_t^1,\widetilde{\pi}_t^2)
+
\eta^{-1}
\operatorname{KL}
(\widetilde{\pi}_t^2,\pi_{\mathrm{ref}}\mid x)
\right].
\label{eq:gp-estimated-best-response}
\end{align}
Combining Eq.~\eqref{eq:gp-new-regret-gap} and
Eq.~\eqref{eq:gp-estimated-best-response} gives
\begin{align}
G_t
\le
\mathbb E_{x\sim d}
\Big[
&
\big(P^\star-\widehat P_t\big)
(x,\widehat{\pi}_t^1,\widehat{\pi}_t^1)-
\big(P^\star-\widehat P_t\big)
(x,\widehat{\pi}_t^1,\widetilde{\pi}_t^2)
\Big].
\label{eq:gp-model-gap}
\end{align}

By the reciprocity identity above, the first term in
Eq.~\eqref{eq:gp-model-gap} is zero. Expanding the remaining term gives
\begin{align}
G_t
\le
\mathbb E_{x\sim d}
\Bigg[
\sum_{a^2\in\mathcal A}
&
\Big(
\widehat{\pi}_t^1(a^2\mid x)
-
\widetilde{\pi}_t^2(a^2\mid x)
\Big)
\mathbb E_{a^1\sim\widehat{\pi}_t^1}
\left[
P^\star(x,a^1,a^2)
-
\widehat P_t(x,a^1,a^2)
\right]
\Bigg].
\label{eq:gp-inner-product}
\end{align}

We next control the policy difference in
Eq.~\eqref{eq:gp-inner-product} using the KL regularization.
For every context $x$, the true regularized best response satisfies
\[
\widetilde{\pi}_t^2(a^2\mid x)
=
\frac{
\pi_{\mathrm{ref}}(a^2\mid x)
\exp(-\eta P^\star(x,\widehat{\pi}_t^1,a^2))
}{
\sum_{a'}
\pi_{\mathrm{ref}}(a'\mid x)
\exp(-\eta P^\star(x,\widehat{\pi}_t^1,a'))
}.
\]
Hence,
\[
P^\star(x,\widehat{\pi}_t^1,a^2)
=
-\eta^{-1}
\log
\frac{
\widetilde{\pi}_t^2(a^2\mid x)
}{
\pi_{\mathrm{ref}}(a^2\mid x)
}
+
C_x,
\]
where $C_x
:=
-\eta^{-1}
\log
\left(
\sum_{a'}
\pi_{\mathrm{ref}}(a'\mid x)
\exp(-\eta P^\star(x,\widehat{\pi}_t^1,a'))
\right)$
is independent of $a^2$. Therefore, for any policy $\pi$,
\begin{align}
&
P^\star(x,\widehat{\pi}_t^1,\pi)
+
\eta^{-1}
\operatorname{KL}
(\pi,\pi_{\mathrm{ref}}\mid x)
=
\eta^{-1}
\operatorname{KL}
(\pi,\widetilde{\pi}_t^2\mid x)
+
C_x.
\label{eq:gp-variational-expansion}
\end{align}
Since the KL divergence is nonnegative and vanishes at
$\pi=\widetilde{\pi}_t^2$, the minimum of the left-hand side of
Eq.~\eqref{eq:gp-variational-expansion} over $\pi$ is $C_x$.
Taking $\pi=\widehat{\pi}_t^1$ therefore gives
\begin{align}
&
P^\star(x,\widehat{\pi}_t^1,\widehat{\pi}_t^1)
+
\eta^{-1}
\operatorname{KL}
(\widehat{\pi}_t^1,\pi_{\mathrm{ref}}\mid x)
-
\min_{\pi}
\left\{
P^\star(x,\widehat{\pi}_t^1,\pi)
+
\eta^{-1}
\operatorname{KL}
(\pi,\pi_{\mathrm{ref}}\mid x)
\right\}
=
\eta^{-1}
\operatorname{KL}
(\widehat{\pi}_t^1,\widetilde{\pi}_t^2\mid x).
\label{eq:gp-gibbs-identity}
\end{align}
Averaging over $x$ and using
Eq.~\eqref{eq:gp-new-regret-gap} yields the exact identity
\begin{equation}
G_t
=
\eta^{-1}
\mathbb E_{x\sim d}
\left[
\operatorname{KL}
(\widehat{\pi}_t^1,\widetilde{\pi}_t^2\mid x)
\right].
\label{eq:gp-kl-identity}
\end{equation}

We now lower bound the KL divergence by a weighted squared distance.
Both $\widehat\pi_t^1$ and $\widetilde\pi_t^2$ are Gibbs policies
induced by scores with range at most one. Hence, for every
$(x,a)$,
\[
\widehat{\pi}_t^1(a\mid x)
\le
e^\eta\pi_{\mathrm{ref}}(a\mid x),
\qquad
\widetilde{\pi}_t^2(a\mid x)
\le
e^\eta\pi_{\mathrm{ref}}(a\mid x).
\]
For $s\in[0,1]$, let $\pi_s(\cdot\mid x)
:=
(1-s)\widetilde{\pi}_t^2(\cdot\mid x)
+
s\widehat{\pi}_t^1(\cdot\mid x)$.
Then
$\pi_s(a\mid x)\le e^\eta\pi_{\mathrm{ref}}(a\mid x)$.
Since the Hessian of the negative entropy
$F(p):=\sum_a p(a)\log p(a)$ is
$\nabla^2 F(p)=\operatorname{diag}(1/p(a))$, we have
\[
\nabla^2 F(\pi_s(\cdot\mid x))
\succeq
e^{-\eta}
\operatorname{diag}
\left(
\frac{1}{\pi_{\mathrm{ref}}(a\mid x)}
\right).
\]
Moreover, since the KL divergence is the Bregman divergence
induced by $F$, its integral second-order representation gives
\begin{align}
&
\operatorname{KL}
(\widehat{\pi}_t^1,\widetilde{\pi}_t^2\mid x)
\nonumber\\
&=
\int_0^1
(1-s)
\big(
\widehat{\pi}_t^1-\widetilde{\pi}_t^2
\big)^\top
\nabla^2 F(\pi_s)
\big(
\widehat{\pi}_t^1-\widetilde{\pi}_t^2
\big)
\,ds
\nonumber\\
&\ge
\frac{e^{-\eta}}{2}
\sum_{a\in\mathcal A}
\frac{
\big(
\widehat{\pi}_t^1(a\mid x)
-
\widetilde{\pi}_t^2(a\mid x)
\big)^2
}{
\pi_{\mathrm{ref}}(a\mid x)
},
\label{eq:weighted-kl}
\end{align}
where we used $\int_0^1(1-s)\,ds=1/2$.
Combining Eq.~\eqref{eq:gp-kl-identity} and
Eq.~\eqref{eq:weighted-kl}, we obtain
\begin{align}
&
\mathbb E_{x\sim d}
\left[
\sum_{a\in\mathcal A}
\frac{
\big(
\widehat{\pi}_t^1(a\mid x)
-
\widetilde{\pi}_t^2(a\mid x)
\big)^2
}{
\pi_{\mathrm{ref}}(a\mid x)
}
\right]\le
2\eta e^\eta G_t.
\label{eq:gp-policy-distance}
\end{align}

Finally, applying weighted Cauchy--Schwarz over $a^2$,
followed by Cauchy--Schwarz over $x$, to
Eq.~\eqref{eq:gp-inner-product} gives
\begin{align}
G_t
\le&
\left(
\mathbb E_{x\sim d}
\left[
\sum_{a^2\in\mathcal A}
\frac{
\big(
\widehat{\pi}_t^1(a^2\mid x)
-
\widetilde{\pi}_t^2(a^2\mid x)
\big)^2
}{
\pi_{\mathrm{ref}}(a^2\mid x)
}
\right]
\right)^{1/2}
\left(
\mathbb E_{\substack{
x\sim d,\,
a^2\sim\pi_{\mathrm{ref}}
}}
\left[
\left(
\mathbb E_{a^1\sim\widehat{\pi}_t^1}
\left[
P^\star(x,a^1,a^2)
-
\widehat P_t(x,a^1,a^2)
\right]
\right)^2
\right]
\right)^{1/2}.
\label{eq:gp-weighted-cs}
\end{align}
Using Eq.~\eqref{eq:gp-policy-distance} for the first factor
and Jensen's inequality for the second factor yields
\begin{align}
G_t
\le&
\sqrt{2\eta e^\eta G_t}
\left(
\mathbb E_{\substack{
x\sim d,\,
a^1\sim\widehat{\pi}_t^1,\,
a^2\sim\pi_{\mathrm{ref}}
}}
\left[
\left(
P^\star(x,a^1,a^2)
-
\widehat P_t(x,a^1,a^2)
\right)^2
\right]
\right)^{1/2}.
\label{eq:gp-final-cs}
\end{align}
If $G_t>0$, dividing both sides by $\sqrt{G_t}$ and squaring gives
\[
G_t
\le
2\eta e^\eta
\mathbb E_{\substack{
x\sim d,\,
a^1\sim\widehat{\pi}_t^1,\,
a^2\sim\pi_{\mathrm{ref}}
}}
\left[
\left(
P^\star(x,a^1,a^2)
-
\widehat P_t(x,a^1,a^2)
\right)^2
\right].
\]
The result is immediate when $G_t=0$.
\end{proof}

\subsubsection{Proof of Lemma \ref{uniformconvergnce_preference}}

\begin{proof}[Proof of Lemma \ref{uniformconvergnce_preference}]
Let $\mathcal H_i^{\operatorname{GP}}
:=
\{(x_j,a_j^1,a_j^2,y_j)\}_{j=1}^i$.
For any fixed $P\in\mathcal P$, define
\[
Y_{P,i}
:=
\left(
P(x_i,a_i^1,a_i^2)-P^\star(x_i,a_i^1,a_i^2)
\right)^2,
\qquad
\mu_{P,i}
:=
\mathbb E[Y_{P,i}\mid\mathcal H_{i-1}^{\operatorname{GP}}].
\]
Since $P$ and $P^\star$ take values in $[0,1]$,
we have $0\le Y_{P,i}\le1$.
Let $c:=1-e^{-1}$.
Using $e^{-y}\le 1-cy$ for $y\in[0,1]$, we obtain
\[
\mathbb E\!\left[
e^{-Y_{P,i}}
\mid\mathcal H_{i-1}^{\operatorname{GP}}
\right]
\le 1-c\mu_{P,i}
\le e^{-c\mu_{P,i}}.
\]
Consequently,
\[
Z_{P,n}
:=
\exp\left(
c\sum_{i=1}^{n}\mu_{P,i}
-
\sum_{i=1}^{n}Y_{P,i}
\right),
\qquad Z_{P,0}:=1,
\]
is a nonnegative supermartingale.
For fixed $P$ and $t$, Markov's inequality therefore gives,
with probability at least $1-\delta_{t,P}$,
\[
\sum_{i=1}^{t-1}\mu_{P,i}
\le
\frac{1}{c}
\left(
\sum_{i=1}^{t-1}Y_{P,i}
+
\log\frac{1}{\delta_{t,P}}
\right)
\le
2\sum_{i=1}^{t-1}Y_{P,i}
+
2\log\frac{1}{\delta_{t,P}},
\]
where the last inequality uses $c^{-1}<2$.

Let $\delta_{t,P}:=\delta/(2N_{\mathcal P}t^3)$.
Taking a union bound over all $P\in\mathcal P$ and
$t=2,\ldots,T$, and using
\[
\sum_{t=2}^{T}\sum_{P\in\mathcal P}\delta_{t,P}
\le \frac{\delta}{2}\sum_{t=2}^{\infty}t^{-3}
\le \frac{\delta}{2},
\]
we conclude that, with probability at least $1-\delta/2$,
simultaneously for all $P\in\mathcal P$ and $t=2,\ldots,T$,
\[
\sum_{i=1}^{t-1}
\mathbb E[Y_{P,i}\mid\mathcal H_{i-1}^{\operatorname{GP}}]
\le
2\sum_{i=1}^{t-1}Y_{P,i}
+
2\log\frac{2N_{\mathcal P}t^3}{\delta}.
\]

Under the sampling rule of our algorithm, conditional on
$\mathcal H_{i-1}^{\operatorname{GP}}$, we have $x_i\sim d$,
and the two actions are sampled independently according to
$a_i^1\sim\widehat\pi_i^1(\cdot\mid x_i)$ and
$a_i^2\sim\pi_{\rm ref}(\cdot\mid x_i)$.
Thus, for every fixed $P\in\mathcal P$,
\[
\mathbb E[Y_{P,i}\mid\mathcal H_{i-1}^{\operatorname{GP}}]
=
\mathbb E_{x\sim d,\,
a^1\sim\widehat\pi_i^1,\,
a^2\sim\pi_{\rm ref}}
\left[
\left(P(x,a^1,a^2)-P^\star(x,a^1,a^2)\right)^2
\right].
\]
Since the preceding event holds uniformly over
$P\in\mathcal P$, we may substitute $P=\widehat P_t$.
Therefore, simultaneously for all $t=2,\ldots,T$,
\begin{align*}
&
\sum_{i=1}^{t-1}
\mathbb E_{x\sim d,\,
a^1\sim\widehat\pi_i^1,\,
a^2\sim\pi_{\rm ref}}
\left[
\left(
\widehat P_t(x,a^1,a^2)-P^\star(x,a^1,a^2)
\right)^2
\right]
\\
&\le
2\sum_{i=1}^{t-1}
\left(
\widehat P_t(x_i,a_i^1,a_i^2)
-
P^\star(x_i,a_i^1,a_i^2)
\right)^2
+
2\log\frac{2N_{\mathcal P}t^3}{\delta}.
\end{align*}

By Lemma \ref{lemma3wu}, applied with confidence level
$\delta/2$, with probability at least $1-\delta/2$,
simultaneously for all $t\in[T]$,
\[
\sum_{i=1}^{t-1}
\left(
\widehat P_t(x_i,a_i^1,a_i^2)
-
P^\star(x_i,a_i^1,a_i^2)
\right)^2
\le
2\log\frac{2N_{\mathcal P}T}{\delta}.
\]
Combining the two high-probability events by a union bound,
with probability at least $1-\delta$, simultaneously for all
$t=2,\ldots,T$,
\begin{align*}
&
\sum_{i=1}^{t-1}
\mathbb E_{x\sim d,\,
a^1\sim\widehat\pi_i^1,\,
a^2\sim\pi_{\rm ref}}
\left[
\left(
\widehat P_t(x,a^1,a^2)-P^\star(x,a^1,a^2)
\right)^2
\right]
\\
&\le
4\log\frac{2N_{\mathcal P}T}{\delta}
+
2\log\frac{2N_{\mathcal P}t^3}{\delta}
\\
&\le
6\log\frac{2N_{\mathcal P}T^3}{\delta}.
\end{align*}
In particular,
\begin{align}\label{ccccc}
\sum_{i=1}^{t-1}
\mathbb E_{x\sim d,\,
a^1\sim\widehat\pi_i^1,\,
a^2\sim\pi_{\rm ref}}
\left[
\left(
\widehat P_t(x,a^1,a^2)-P^\star(x,a^1,a^2)
\right)^2
\right]
\le
6\log\frac{2N_{\mathcal P}T^3}{\delta}.
\end{align}
This proves the claim.
\end{proof}

\subsubsection{UCB-Based Exploration for the GP Model}

We next consider an uncertainty-based variant of \texttt{ORLHF-GS} that
explicitly explores uncertain preference comparisons. The MLE
$\widehat P_t$ and the learned policy $\widehat\pi_t^1$ are constructed
in the same way as in Algorithm~\ref{alg2}. The only modification is the
sampling rule for the second action. After observing $x_t$ and sampling
$a_t^1 \sim \widehat\pi_t^1(\cdot\mid x_t)$, we select
\begin{equation}
    a_t^2
    \in
    \arg\max_{a^2\in\mathcal A}
    \min\left\{
        1,
        U_{\operatorname{GP}}^2
        \left(
            \lambda,x_t,a_t^1,a^2,\mathcal P;
            \mathcal D_{t-1}^{\operatorname{GP}}
        \right)
    \right\}.
    \label{eq:gp_ucb_sampling}
\end{equation}
The resulting action pair $(a_t^1,a_t^2)$ is then used to obtain the
preference feedback and update the MLE. We refer to this sampling rule
as \texttt{GP-UCB}. Unlike directly adding an optimistic bonus to the
estimated preference function, the rule in
Eq.~\eqref{eq:gp_ucb_sampling} leaves the reciprocal structure of
$\widehat P_t$ unchanged and uses uncertainty only for data collection.

\begin{proof}[Proof of Corollary~\ref{cor:gp_ucb}]
Let $G_t
    :=
    J_{\operatorname{GP}}^\star
    -
    J_{\operatorname{GP}}
    \left(
        \widehat\pi_t^1,
        \widetilde\pi_t^2
    \right)$
denote the instantaneous regret. Recall from the proof of
Lemma~\ref{lem:gp_regret_decomposition} that
\begin{align}
    G_t
    \le
    \mathbb E_{x\sim d}
    \Bigg[
        \sum_{a^2\in\mathcal A}
        \left(
            \widehat\pi_t^1(a^2\mid x)
            -
            \widetilde\pi_t^2(a^2\mid x)
        \right)
        \mathbb E_{a^1\sim\widehat\pi_t^1}
        \left[
            P^\star(x,a^1,a^2)
            -
            \widehat P_t(x,a^1,a^2)
        \right]
    \Bigg].
    \label{eq:gp_ucb_gap}
\end{align}
Moreover, the same proof gives the exact identity
\begin{equation}
    G_t
    =
    \eta^{-1}
    \mathbb E_{x\sim d}
    \left[
        \operatorname{KL}
        \left(
            \widehat\pi_t^1,
            \widetilde\pi_t^2
            \mid x
        \right)
    \right].
    \label{eq:gp_ucb_kl_identity}
\end{equation}

\paragraph{Logarithmic bound.}
By Eq.~\eqref{eq:gp_ucb_gap} and the inequality
\[
    \left|
        \sum_{a^2\in\mathcal A}
        (p(a^2)-q(a^2))f(a^2)
    \right|
    \le
    \|p-q\|_1
    \max_{a^2\in\mathcal A}|f(a^2)|,
\]
we have
\begin{align}\label{nishiyige}
    G_t
    &\le
    \mathbb E_{x\sim d}
    \Bigg[
        \left\|
            \widehat\pi_t^1(\cdot\mid x)
            -
            \widetilde\pi_t^2(\cdot\mid x)
        \right\|_1
        \max_{a^2\in\mathcal A}
        \left|
            \mathbb E_{a^1\sim\widehat\pi_t^1}
            \left[
                P^\star(x,a^1,a^2)
                -
                \widehat P_t(x,a^1,a^2)
            \right]
        \right|
    \Bigg].
\end{align}
By Pinsker's inequality,
\[
    \left\|
        \widehat\pi_t^1(\cdot\mid x)
        -
        \widetilde\pi_t^2(\cdot\mid x)
    \right\|_1
    \le
    \sqrt{
        2\operatorname{KL}
        \left(
            \widehat\pi_t^1,
            \widetilde\pi_t^2
            \mid x
        \right)
    }.
\]
Therefore, applying Cauchy--Schwarz with respect to $x\sim d$,
\begin{align*}
    G_t
    &\le
    \sqrt{
        2
        \mathbb E_{x\sim d}
        \left[
            \operatorname{KL}
            \left(
                \widehat\pi_t^1,
                \widetilde\pi_t^2
                \mid x
            \right)
        \right]
    }
    \sqrt{
        \mathbb E_{x\sim d}
        \left[
            \max_{a^2\in\mathcal A}
            \left(
                \mathbb E_{a^1\sim\widehat\pi_t^1}
                \left[
                    P^\star(x,a^1,a^2)
                    -
                    \widehat P_t(x,a^1,a^2)
                \right]
            \right)^2
        \right]
    }.
\end{align*}
Using Eq.~\eqref{eq:gp_ucb_kl_identity}, we obtain
\begin{align*}
    G_t
    &\le
    \sqrt{2\eta G_t}
    \sqrt{
        \mathbb E_{x\sim d}
        \left[
            \max_{a^2\in\mathcal A}
            \left(
                \mathbb E_{a^1\sim\widehat\pi_t^1}
                \left[
                    P^\star(x,a^1,a^2)
                    -
                    \widehat P_t(x,a^1,a^2)
                \right]
            \right)^2
        \right]
    }.
\end{align*}
If $G_t=0$, the desired bound is immediate. Otherwise, dividing by
$\sqrt{G_t}$ and squaring both sides gives
\begin{align*}
    G_t
    \le
    2\eta
    \mathbb E_{x\sim d}
    \left[
        \max_{a^2\in\mathcal A}
        \left(
            \mathbb E_{a^1\sim\widehat\pi_t^1}
            \left[
                P^\star(x,a^1,a^2)
                -
                \widehat P_t(x,a^1,a^2)
            \right]
        \right)^2
    \right].
\end{align*}
Finally, by Jensen's inequality,
\[
    \left(
        \mathbb E_{a^1\sim\widehat\pi_t^1}
        \left[
            P^\star(x,a^1,a^2)
            -
            \widehat P_t(x,a^1,a^2)
        \right]
    \right)^2
    \le
    \mathbb E_{a^1\sim\widehat\pi_t^1}
    \left[
        \left(
            P^\star(x,a^1,a^2)
            -
            \widehat P_t(x,a^1,a^2)
        \right)^2
    \right].
\]
Hence,
\begin{align}
    G_t
    \le
    2\eta
    \mathbb E_{x\sim d}
    \left[
        \max_{a^2\in\mathcal A}
        \mathbb E_{a^1\sim\widehat\pi_t^1}
        \left[
            \left(
                P^\star(x,a^1,a^2)
                -
                \widehat P_t(x,a^1,a^2)
            \right)^2
        \right]
    \right].
    \label{eq:gp_ucb_fast_step}
\end{align}

By the MLE confidence event used in the
eluder-dimension-dependent proof of Theorem~\ref{theorem2},
with confidence level $\delta/2$, simultaneously for all $t\in[T]$,
and for $\lambda\le\log(2N_{\mathcal P}T/\delta)$,
\begin{align}
    &\left(
        P^\star(x,a^1,a^2)
        -
        \widehat P_t(x,a^1,a^2)
    \right)^2\le
    3\log\frac{2N_{\mathcal P}T}{\delta}
    \min\left\{
        1,
        U_{\operatorname{GP}}^2
        \left(
            \lambda,x,a^1,a^2,\mathcal P;
            \mathcal D_{t-1}^{\operatorname{GP}}
        \right)
    \right\}.
    \label{eq:gp_ucb_confidence}
\end{align}
Combining Eq.~\eqref{eq:gp_ucb_fast_step} and
Eq.~\eqref{eq:gp_ucb_confidence}, and using
\begin{align*}
    &\max_{a^2\in\mathcal A}
    \mathbb E_{a^1\sim\widehat\pi_t^1}
    \left[
        \min\left\{
            1,
            U_{\operatorname{GP}}^2
            \left(
                \lambda,x,a^1,a^2,\mathcal P;
                \mathcal D_{t-1}^{\operatorname{GP}}
            \right)
        \right\}
    \right]\le
    \mathbb E_{a^1\sim\widehat\pi_t^1}
    \left[
        \max_{a^2\in\mathcal A}
        \min\left\{
            1,
            U_{\operatorname{GP}}^2
            \left(
                \lambda,x,a^1,a^2,\mathcal P;
                \mathcal D_{t-1}^{\operatorname{GP}}
            \right)
        \right\}
    \right],
\end{align*}
we obtain
\begin{align}
    G_t
    =
    O\Bigg(
        \eta
        \log\frac{N_{\mathcal P}T}{\delta}
        \,
        \mathbb E_{\substack{
            x\sim d,\,
            a^1\sim\widehat\pi_t^1
        }}
        \Bigg[
            \max_{a^2\in\mathcal A}
            \min\Bigg\{
                1,
                U_{\operatorname{GP}}^2
                \left(
                    \lambda,x,a^1,a^2,\mathcal P;
                 \mathcal D_{t-1}^{\operatorname{GP}}
                \right)
            \Bigg\}
        \Bigg]
    \Bigg).
    \label{eq:gp_ucb_fast_inst}
\end{align}

Since \texttt{GP-UCB} selects $a_t^2$ according to
Eq.~\eqref{eq:gp_ucb_sampling}, conditioned on the history before round
$t$, the expectation in Eq.~\eqref{eq:gp_ucb_fast_inst} is exactly the conditional expectation of
$\min\left\{
        1,
        U_{\operatorname{GP}}^2
        \left(
            \lambda,x_t,a_t^1,a_t^2,\mathcal P;
            \mathcal D_{t-1}^{\operatorname{GP}}
        \right)
    \right\}$.
Moreover, for every realized sequence
$\{(x_t,a_t^1,a_t^2)\}_{t=1}^T$, the definition of
$d_{\operatorname{GP}}(\lambda,\mathcal P,T)$ gives
\begin{align*}
    \sum_{t=1}^T
    \min\left\{
        1,
        U_{\operatorname{GP}}^2
        \left(
            \lambda,x_t,a_t^1,a_t^2,\mathcal P;
           \mathcal D_{t-1}^{\operatorname{GP}}
        \right)
    \right\}
    \le
    d_{\operatorname{GP}}(\lambda,\mathcal P,T).
\end{align*}
Therefore, applying the same predictable-to-realized concentration
argument as in the eluder-dimension-dependent proof of
Theorem~\ref{theorem2}, with confidence level $\delta/2$, gives
\begin{align}
    \sum_{t=1}^T
    \mathbb E_{\substack{
        x\sim d,\,
        a^1\sim\widehat\pi_t^1
    }}
    \Bigg[
        \max_{a^2\in\mathcal A}
        \min\Bigg\{
            1,
            U_{\operatorname{GP}}^2
            \left(
                \lambda,x,a^1,a^2,\mathcal P;
             \mathcal D_{t-1}^{\operatorname{GP}}
            \right)
        \Bigg\}
    \Bigg]
    =
    O\left(
        d_{\operatorname{GP}}(\lambda,\mathcal P,T)
        +
        \log\frac{1}{\delta}
    \right).
    \label{eq:gp_ucb_eluder}
\end{align}
Taking a union bound over the MLE confidence event and the
predictable-to-realized concentration event, the preceding bounds
hold simultaneously with probability at least $1-\delta$. Therefore,
\begin{equation}
    \operatorname{Reg}_{\operatorname{GP}}(T)
    =
    O\left(
        \eta
        \left(
            d_{\operatorname{GP}}(\lambda,\mathcal P,T)
            +
            \log\frac{1}{\delta}
        \right)
        \log\frac{N_{\mathcal P}T}{\delta}
    \right).
    \label{eq:gp_ucb_fast_rate}
\end{equation}

\paragraph{$\eta$-independent bound.}
Returning to Eq.~\eqref{nishiyige} and using $\left\|
        \widehat\pi_t^1(\cdot\mid x)
        -
        \widetilde\pi_t^2(\cdot\mid x)
    \right\|_1
    \le 2$,
we obtain
\begin{align*}
    G_t
    &\le
    2
    \mathbb E_{x\sim d}
    \left[
        \max_{a^2\in\mathcal A}
        \left|
            \mathbb E_{a^1\sim\widehat\pi_t^1}
            \left[
                P^\star(x,a^1,a^2)
                -
                \widehat P_t(x,a^1,a^2)
            \right]
        \right|
    \right]
    \\
    &\le
    2
    \sqrt{
        \mathbb E_{x\sim d}
        \left[
            \max_{a^2\in\mathcal A}
            \left(
                \mathbb E_{a^1\sim\widehat\pi_t^1}
                \left[
                    P^\star(x,a^1,a^2)
                    -
                    \widehat P_t(x,a^1,a^2)
                \right]
            \right)^2
        \right]
    }
    \\
    &\le
    2
    \sqrt{
        3\log\frac{2N_{\mathcal P}T}{\delta}
    }
    \sqrt{
        \mathbb E_{\substack{
            x\sim d,\,
            a^1\sim\widehat\pi_t^1
        }}
        \left[
            \max_{a^2\in\mathcal A}
            \min\left\{
                1,
                U_{\operatorname{GP}}^2
                \left(
                    \lambda,x,a^1,a^2,\mathcal P;
             \mathcal D_{t-1}^{\operatorname{GP}}
                \right)
            \right\}
        \right]
    }.
\end{align*}
Summing over $t$ and applying Cauchy--Schwarz together with
Eq.~\eqref{eq:gp_ucb_eluder} gives
\begin{align*}
    \operatorname{Reg}_{\operatorname{GP}}(T)
    &=
    O\left(
        \sqrt{
            T
            \left(
                d_{\operatorname{GP}}(\lambda,\mathcal P,T)
                +
                \log\frac{1}{\delta}
            \right)
            \log\frac{N_{\mathcal P}T}{\delta}
        }
    \right).
\end{align*}
Taking the minimum of the two bounds completes the proof.
\end{proof}

\subsection{Bradley-Terry Model}

The proof of Theorem~\ref{theorem3} relies on the following two key lemmas.
Lemma~\ref{regretdecompoBT} establishes the instantaneous KL-regularized regret decomposition under the BT model, while Lemma~\ref{uniformconvergenceBT} provides the uniform convergence guarantee for the MLE estimator under the BT model.

\begin{lemma}[Instantaneous Regret Decomposition, Bradley-Terry Model]\label{regretdecompoBT}
For all \(t\in[T]\), the instantaneous regret in the BT model satisfies
\[
J_{\operatorname{BT}}(\pi^\star_{\operatorname{BT}})-J_{\operatorname{BT}}(\widehat{\pi}^1_t)
\le
\eta
\mathbb{E}_{x\sim d,\,
a^1\sim\pi_t',\,
a^2\sim\pi_{\operatorname{ref}}}
\!\left[
\left(
\left(R^\star(x,a^1)-R^\star(x,a^2)\right)
-
\left(\widehat{R}_t(x,a^1)-\widehat{R}_t(x,a^2)\right)
\right)^2
\right],
\]
where $\pi_t'$ denotes the Gibbs policy induced by some
$[0,1]$-valued reward function $R_t'$.

\end{lemma}

\begin{lemma}[Uniform Prediction Error Bound, Bradley-Terry Model]\label{uniformconvergenceBT} Suppose Assumption~\ref{assumption1} holds.
Let $\widehat R_t$ be the MLE estimator over the function class $\mathcal R$
constructed from the samples
$\{(x_i,a_i^1,a_i^2,y_i)\}_{i=1}^{t-1}$,
where $x_i\sim d$,
$a_i^1\sim\widehat\pi_i^1(\cdot\mid x_i)$, and
$a_i^2\sim\pi_{\operatorname{ref}}(\cdot\mid x_i)$ for each $i$.
Then for any such policy sequence
$\{\widehat\pi_i^1\}_{i\ge1}$ and any $\delta\in(0,1)$,
with probability at least $1-\delta$, the following holds
simultaneously for all $t = 2,\dots,T$:
\begin{align}
\begin{split}
\nonumber
\sum_{i=1}^{t-1}
\mathbb E_{x\sim d,\;a^1\sim\widehat\pi_i^1,\;a^2\sim\pi_{\rm ref}}
\left[
\left(
\left(R^\star(x,a^1)-R^\star(x,a^2)\right)-\left(\widehat{R}_t(x,a^1)-\widehat{R}_t(x,a^2)\right)
\right)^2
\right]
\le
&24e^2\log\frac{2N_{\mathcal R}T^3}{\delta}.
\end{split}
\end{align}
\end{lemma}

\begin{proof}[Proof of Theorem~\ref{theorem3}]
The proof follows the same argument as those of
Theorem~\ref{theorem1} and Theorem~\ref{theorem2}.
For each \(t\in[T]\), define
$S_t
:=
\mathbb E_{x\sim d,\,
a^1\sim\pi_t',\,
a^2\sim\pi_{\operatorname{ref}}}
\!\left[
\left(
\left(R^\star(x,a^1)-R^\star(x,a^2)\right)
-
\left(\widehat R_t(x,a^1)-\widehat R_t(x,a^2)\right)
\right)^2
\right]$.
Applying Lemma~\ref{regretdecompoBT} yields the regret decomposition
\begin{align}
\label{regretdecompositionBT}
\operatorname{Reg}_{\mathrm{BT}}(T)
&=
\sum_{t=1}^T
\Big(
J_{\mathrm{BT}}(\pi^\star_{\text{BT}})
-
J_{\mathrm{BT}}(\widehat\pi_t^1)
\Big)\le
\eta
\sum_{t=1}^T S_t .
\end{align}
We next bound $\sum_{t=1}^T S_t$. By Lemma~\ref{uniformconvergenceBT} and the same Gibbs-policy comparison argument as in the proof of
Theorem~\ref{theorem1}, we obtain
\begin{align}
\label{Sbound3}
S_t
\le
\frac{
24e^{2+2\eta}
}{
t-1
}
\log\frac{2N_{\mathcal R}T^3}{\delta},
\qquad
\forall t=2,\dots,T.
\end{align}
For \(t=1\), since $R^\star,\widehat R_1$ are $[0,1]$-valued,
we trivially have \(S_1\le4\). Therefore,
\[
\sum_{t=1}^T S_t
\le
4
+
24e^{2+2\eta}
\log\frac{2N_{\mathcal R}T^3}{\delta}
\sum_{t=2}^T\frac1{t-1}.
\]
Using
\(\sum_{t=2}^T(t-1)^{-1}\le1+\log T\), we obtain
\[
\sum_{t=1}^T S_t
=
O\!\left(
e^{2\eta}
\log T
\log\frac{N_{\mathcal R}T}{\delta}
\right).
\]
Substituting this bound into
Eq.~\eqref{regretdecompositionBT} yields
\[
\operatorname{Reg}_{\mathrm{BT}}(T)
=
O\!\left(
\eta e^{2\eta}
\log T
\log\frac{N_{\mathcal R}T}{\delta}
\right).
\]
This completes the proof.
\end{proof}

\subsubsection{Proof of Lemma~\ref{regretdecompoBT}}

\begin{proof}[Proof of Lemma~\ref{regretdecompoBT}]
 The selected policy in round $t$ satisfies
\begin{align*}
\widehat{\pi}^1_t
&=
\arg\max_{\pi}
\mathbb{E}_{x\sim d,\;a\sim\pi}
\!\left[
\widehat{R}_t(x,a)-\eta^{-1}\mathrm{KL}(\pi,\pi_\text{ref}\mid x)
\right] \\
&=
\arg\max_{\pi}
\mathbb{E}_{x\sim d,\;a\sim\pi}
\!\left[
\widehat{R}_t(x,a)+l(x)-\eta^{-1}\mathrm{KL}(\pi,\pi_\text{ref}\mid x)
\right],
\end{align*}
for any function $l:\X \rightarrow \mathbb R$.
By selecting 
$l_t(x)
:=
\mathbb{E}_{a'\sim\pi_\text{ref}}
\!\left[
R^\star(x,a') - \widehat{R}_t(x,a')
\right]$, we have for all $t \in [T]$,
\begin{align*}
J_{\text{BT}}(\pi^\star_{\text{BT}})
-
J_{\text{BT}}(\widehat{\pi}^1_t)
&\le
\eta
\mathbb{E}_{x\sim d,\;a\sim\pi'_{t}}
\!\left[
\left(
R^\star(x,a) - \widehat{R}_t(x,a)-l_t(x)
\right)^2
\right] \\
&=
\eta
\mathbb{E}_{x\sim d,\,
a^1\sim\pi'_{t}}
\!\left[
\left(R^\star(x,a^1)-\widehat{R}_t(x,a^1)
-
\mathbb{E}_{a^2\sim\pi_\text{ref}}
\left(
R^\star(x,a^2)-\widehat{R}_t(x,a^2)
\right)
\right)^2
\right] \\
&\le
\eta
\mathbb{E}_{x\sim d,\,
a^1\sim\pi'_{t},\,
a^2\sim\pi_\text{ref}}
\!\left[
\left(
\left(R^\star(x,a^1)-R^\star(x,a^2)\right)
-
\left(\widehat{R}_t(x,a^1)-\widehat{R}_t(x,a^2)\right)
\right)^2
\right].
\end{align*}
To justify the first inequality, although
$\widehat R_t+l_t$ need not be $[0,1]$-valued, for any
$\gamma\in[0,1]$ we have
\[
\gamma(\widehat R_t+l_t)+(1-\gamma)R^\star
=
\gamma\widehat R_t+(1-\gamma)R^\star+\gamma l_t.
\]
Since $\gamma l_t(x)$ is independent of the action, the
right-hand side induces the same Gibbs policy as the
$[0,1]$-valued reward function
$\gamma\widehat R_t+(1-\gamma)R^\star$.
Therefore, repeating the mean-value argument in the proof of
Lemma~\ref{lemma3} yields the first inequality for some
$\gamma_t\in[0,1]$, where $\pi_t'$ is the Gibbs policy induced
by $R_t'
=
\gamma_t\widehat R_t+(1-\gamma_t)R^\star$.
The last inequality follows from Jensen's inequality.
\end{proof}

\subsubsection{Proof of Lemma~\ref{uniformconvergenceBT}}
\begin{proof}[Proof of Lemma~\ref{uniformconvergenceBT}]
Substituting
\(
P(x_i,a_i^1,a_i^2)
\)
with
\(
\sigma\!\left(R(x_i,a_i^1)-R(x_i,a_i^2)\right)
\)
and
\(
P^\star(x_i,a_i^1,a_i^2)
\)
with
$\sigma\!\left(R^\star(x_i,a_i^1)-R^\star(x_i,a_i^2)\right)$
in the proof of Lemma~\ref{uniformconvergnce_preference}, we obtain
\begin{align*}
&
\sum_{i=1}^{t-1}
\mathbb E_{x\sim d,\;a^1\sim\widehat\pi_i^1,\;a^2\sim\pi_{\rm ref}}
\left[
\left(
\bigl(R^\star(x,a^1)-R^\star(x,a^2)\bigr)
-
\bigl(\widehat R_t(x,a^1)-\widehat R_t(x,a^2)\bigr)
\right)^2
\right]
\\
\le\;&
4e^2
\sum_{i=1}^{t-1}
\mathbb E_{x\sim d,\;a^1\sim\widehat\pi_i^1,\;a^2\sim\pi_{\rm ref}}
\left[
\left(
\sigma\!\left(R^\star(x,a^1)-R^\star(x,a^2)\right)
-
\sigma\!\left(\widehat R_t(x,a^1)-\widehat R_t(x,a^2)\right)
\right)^2
\right]
\\
\le\;&
24e^2
\log\frac{2N_{\mathcal R}T^3}{\delta}.
\end{align*}
The first inequality follows from the inverse Lipschitz property of the sigmoid function on $[-1,1]$, namely,
$|u-v|
\le
2e\,|\sigma(u)-\sigma(v)|$
for all \(u,v\in[-1,1]\). The second inequality follows directly from Lemma \ref{uniformconvergnce_preference}. This completes the proof.
\end{proof}

\subsubsection{UCB-Based Exploration for the BT Model}

We next consider an uncertainty-based variant of \texttt{ORLHF-GS} under the
BT model. The MLE $\widehat R_t$ and the learned Gibbs policy
$\widehat\pi_t^1$ are constructed in the same way as in Algorithm~\ref{alg2}.
The only modification is the sampling rule for the second action.
After observing $x_t$ and sampling
$a_t^1\sim\widehat\pi_t^1(\cdot\mid x_t)$, we select
\begin{equation}
    a_t^2
    \in
    \arg\max_{a^2\in\mathcal A}
    \min\left\{
        1,
        U_{\operatorname{BT}}^2
        \left(
            \lambda,x_t,a_t^1,a^2,\mathcal R;
            \mathcal D_{t-1}^{\operatorname{BT}}
        \right)
    \right\}.
    \label{eq:bt_ucb_sampling}
\end{equation}
The resulting action pair $(a_t^1,a_t^2)$ is then used to obtain the
preference feedback and update the MLE. We refer to this sampling rule
as \texttt{BT-UCB}.

\begin{proof}[Proof of Corollary~\ref{cor:bt_ucb}]
Let $G_t
    :=
    J_{\operatorname{BT}}(\pi^\star_{\text{BT}})
    -
    J_{\operatorname{BT}}(\widehat\pi_t^1)$
denote the instantaneous regret. Since $J_{\operatorname{BT}}=J_{\operatorname{RF}}$, $\pi^\star_{\text{BT}}$
is the Gibbs policy induced by $R^\star$, and $\widehat\pi_t^1$
is the Gibbs policy induced by $\widehat R_t$. Following the same
calculation as in the proof of Lemma~\ref{lemma3}, we have 
\begin{align}
    G_t
    &=
    \eta^{-1}
    \mathbb E_{x\sim d}
    \left[
        \operatorname{KL}
        \left(
            \widehat\pi_t^1,
            \pi^\star_{\text{BT}}
            \mid x
        \right)
    \right]
    \nonumber\\
    &=
    \eta^{-1}
    \mathbb E_{x\sim d}
    \Bigg[
        \log\Bigg(
        \mathbb E_{a\sim\widehat\pi_t^1}
        \left[
            \exp\left(
                \eta
                \left(
                    R^\star(x,a)
                    -
                    \widehat R_t(x,a)
                \right)
            \right)
        \right]\Bigg)  
        -
        \eta
        \mathbb E_{a\sim\widehat\pi_t^1}
        \left[
            R^\star(x,a)
            -
            \widehat R_t(x,a)
        \right]
    \Bigg].
    \label{eq:bt_ucb_kl_identity}
\end{align}

\paragraph{Logarithmic bound.}
By Hoeffding's lemma, for every $x\in\mathcal X$,
\begin{align*}
    &\log \Bigg(
    \mathbb E_{a\sim\widehat\pi_t^1}
    \left[
        \exp\left(
            \eta
            \left(
                R^\star(x,a)
                -
                \widehat R_t(x,a)
            \right)
        \right)
    \right]\Bigg)
    -
    \eta
    \mathbb E_{a\sim\widehat\pi_t^1}
    \left[
        R^\star(x,a)
        -
        \widehat R_t(x,a)
    \right]
    \\
    &\le
    \frac{\eta^2}{8}
    \Bigg(
        \max_{a\in\mathcal A}
        \left(
            R^\star(x,a)-\widehat R_t(x,a)
        \right)
        -
        \min_{a\in\mathcal A}
        \left(
            R^\star(x,a)-\widehat R_t(x,a)
        \right)
    \Bigg)^2.
\end{align*}
Hence,
\begin{align}
    G_t
    \le
    \frac{\eta}{8}
    \mathbb E_{x\sim d}
    \Bigg[
        \Bigg(
            \max_{a\in\mathcal A}
            \left(
                R^\star(x,a)-\widehat R_t(x,a)
            \right)
            -
            \min_{a\in\mathcal A}
            \left(
                R^\star(x,a)-\widehat R_t(x,a)
            \right)
        \Bigg)^2
    \Bigg].
    \label{eq:bt_ucb_range}
\end{align}

For any fixed $(x,a^1)$, since
$R^\star(x,a^1)-\widehat R_t(x,a^1)$ lies between the maximum and
minimum appearing in Eq.~\eqref{eq:bt_ucb_range},
\begin{align*}
    &\Bigg(
        \max_{a\in\mathcal A}
        \left(
            R^\star(x,a)-\widehat R_t(x,a)
        \right)
        -
        \min_{a\in\mathcal A}
        \left(
            R^\star(x,a)-\widehat R_t(x,a)
        \right)
    \Bigg)^2
    \\
    &\qquad\le
    4
    \max_{a^2\in\mathcal A}
    \Big(
        R^\star(x,a^1)
        -
        R^\star(x,a^2)
        -
        \widehat R_t(x,a^1)
        +
        \widehat R_t(x,a^2)
    \Big)^2.
\end{align*}
Averaging over
$a^1\sim\widehat\pi_t^1(\cdot\mid x)$ and substituting into
Eq.~\eqref{eq:bt_ucb_range} gives
\begin{align}
    G_t
    \le
    \frac{\eta}{2}
    \mathbb E_{\substack{
        x\sim d,\,
        a^1\sim\widehat\pi_t^1
    }}
    \Bigg[
        \max_{a^2\in\mathcal A}
        \Big(
            R^\star(x,a^1)
            -
            R^\star(x,a^2)
            -
            \widehat R_t(x,a^1)
            +
            \widehat R_t(x,a^2)
        \Big)^2
    \Bigg].
    \label{eq:bt_ucb_on_policy}
\end{align}

Applying Lemma~\ref{lemma3wu} to the preference functions
$\sigma(R(x,a^1)-R(x,a^2))$, $R\in\mathcal R$, and using the inverse
Lipschitz property of the sigmoid function on $[-1,1]$, with
probability at least $1-\delta/2$, simultaneously for all $t\in[T]$,
\begin{align}
    \sum_{i=1}^{t-1}
    \Big(
        R^\star(x_i,a_i^1)
        -
        R^\star(x_i,a_i^2)
        -
        \widehat R_t(x_i,a_i^1)
        +
        \widehat R_t(x_i,a_i^2)
    \Big)^2
    \le
    8e^2
    \log\frac{2N_{\mathcal R}T}{\delta}.
    \label{eq:bt_ucb_empirical}
\end{align}
Since $R^\star,\widehat R_t\in\mathcal R$, the definition of
$U_{\operatorname{BT}}$ and Eq.~\eqref{eq:bt_ucb_empirical} imply
\begin{align*}
    &\Big|
        R^\star(x,a^1)
        -
        R^\star(x,a^2)
        -
        \widehat R_t(x,a^1)
        +
        \widehat R_t(x,a^2)
    \Big|\le
    U_{\operatorname{BT}}
    \left(
        \lambda,x,a^1,a^2,\mathcal R;
        \mathcal D_{t-1}^{\operatorname{BT}}
    \right)
    \sqrt{
        \lambda
        +
        8e^2
        \log\frac{2N_{\mathcal R}T}{\delta}
    }.
\end{align*}
Therefore, for
$\lambda\le4e^2\log(2N_{\mathcal R}T/\delta)$, and using the fact that
the absolute pairwise reward-difference error is at most $2$,
\begin{align}
    &\Big(
        R^\star(x,a^1)
        -
        R^\star(x,a^2)
        -
        \widehat R_t(x,a^1)
        +
        \widehat R_t(x,a^2)
    \Big)^2\le
    12e^2
    \log\frac{2N_{\mathcal R}T}{\delta}
    \min\left\{
        1,
        U_{\operatorname{BT}}^2
        \left(
            \lambda,x,a^1,a^2,\mathcal R;
           \mathcal D_{t-1}^{\operatorname{BT}}
        \right)
    \right\}.
    \label{eq:bt_ucb_confidence}
\end{align}
Combining Eq.~\eqref{eq:bt_ucb_on_policy} and
Eq.~\eqref{eq:bt_ucb_confidence} gives
\begin{align}
    G_t
    \le
    6e^2\eta
    \log\frac{2N_{\mathcal R}T}{\delta}
    \mathbb E_{\substack{
        x\sim d,\,
        a^1\sim\widehat\pi_t^1
    }}
    \Bigg[
        \max_{a^2\in\mathcal A}
        \min\Bigg\{
            1,
            U_{\operatorname{BT}}^2
            \left(
                \lambda,x,a^1,a^2,\mathcal R;
                \mathcal D_{t-1}^{\operatorname{BT}}
            \right)
        \Bigg\}
    \Bigg].
    \label{eq:bt_ucb_fast_inst}
\end{align}

Since \texttt{BT-UCB} selects $a_t^2$ according to
Eq.~\eqref{eq:bt_ucb_sampling}, conditioned on the history before round
$t$, the expectation in Eq.~\eqref{eq:bt_ucb_fast_inst} is exactly the
conditional expectation of
$\min\left\{
        1,
        U_{\operatorname{BT}}^2
        \left(
            \lambda,x_t,a_t^1,a_t^2,\mathcal R;
            \mathcal D_{t-1}^{\operatorname{BT}}
        \right)
    \right\}$.
Moreover, since the definition of
$d_{\operatorname{BT}}(\lambda,\mathcal R,T)$ takes the supremum over
all context-action-pair sequences, for every realized sequence
$\{(x_t,a_t^1,a_t^2)\}_{t=1}^T$,
\begin{align*}
    \sum_{t=1}^T
    \min\left\{
        1,
        U_{\operatorname{BT}}^2
        \left(
            \lambda,x_t,a_t^1,a_t^2,\mathcal R;
            \mathcal D_{t-1}^{\operatorname{BT}}
        \right)
    \right\}
    \le
    d_{\operatorname{BT}}(\lambda,\mathcal R,T).
\end{align*}
Applying the same predictable-to-realized concentration argument as
in the proof of Corollary~\ref{cor:gp_ucb}, with confidence level
$\delta/2$, gives
\begin{align}
    \sum_{t=1}^T
    \mathbb E_{\substack{
        x\sim d,\,
        a^1\sim\widehat\pi_t^1
    }}
    \Bigg[
        \max_{a^2\in\mathcal A}
        \min\Bigg\{
            1,
            U_{\operatorname{BT}}^2
            \left(
                \lambda,x,a^1,a^2,\mathcal R;
                \mathcal D_{t-1}^{\operatorname{BT}}
            \right)
        \Bigg\}
    \Bigg]
    =
    O\left(
        d_{\operatorname{BT}}(\lambda,\mathcal R,T)
        +
        \log\frac{1}{\delta}
    \right).
    \label{eq:bt_ucb_eluder}
\end{align}
Taking a union bound over the MLE confidence event and the
predictable-to-realized concentration event, with probability at
least $1-\delta$,
\begin{align}
    \operatorname{Reg}_{\operatorname{BT}}(T)
    =
    O\left(
        \eta
        \left(
            d_{\operatorname{BT}}(\lambda,\mathcal R,T)
            +
            \log\frac{1}{\delta}
        \right)
        \log\frac{N_{\mathcal R}T}{\delta}
    \right).
\nonumber
\end{align}

\paragraph{$\eta$-independent bound.}
By the optimality of $\widehat\pi_t^1$ under $\widehat R_t$, we have
\begin{align*}
    &\mathbb E_{a\sim\pi^\star_{\text{BT}}}
    \left[
        \widehat R_t(x,a)
    \right]
    -
    \eta^{-1}
    \operatorname{KL}
    \left(
        \pi^\star_{\text{BT}},\pi_{\rm ref}\mid x
    \right)
\le
    \mathbb E_{a\sim\widehat\pi_t^1}
    \left[
        \widehat R_t(x,a)
    \right]
    -
    \eta^{-1}
    \operatorname{KL}
    \left(
        \widehat\pi_t^1,\pi_{\rm ref}\mid x
    \right).
\end{align*}
Therefore, adding and subtracting $\widehat R_t$ in the
KL-regularized objective gives
\begin{align*}
    G_t
    &\le
    \mathbb E_{x\sim d}
    \Bigg[
        \mathbb E_{a\sim\pi^\star_{\text{BT}}}
        \left[
            R^\star(x,a)-\widehat R_t(x,a)
        \right]
        -
        \mathbb E_{a\sim\widehat\pi_t^1}
        \left[
            R^\star(x,a)-\widehat R_t(x,a)
        \right]
    \Bigg]
    \\
    &\le
    \mathbb E_{x\sim d}
    \Bigg[
        \max_{a\in\mathcal A}
        \left(
            R^\star(x,a)-\widehat R_t(x,a)
        \right)
        -
        \min_{a\in\mathcal A}
        \left(
            R^\star(x,a)-\widehat R_t(x,a)
        \right)
    \Bigg],
\end{align*}
where the second inequality follows since the expectation of
$R^\star(x,a)-\widehat R_t(x,a)$ under any policy lies between its
minimum and maximum over $\mathcal A$.

For any fixed $(x,a^1)$,
$R^\star(x,a^1)-\widehat R_t(x,a^1)$ lies between the same minimum
and maximum. Hence, its distance from at least one of the two
endpoints is at least half of the range, which implies
\begin{align*}
    &\max_{a\in\mathcal A}
    \left(
        R^\star(x,a)-\widehat R_t(x,a)
    \right)
    -
    \min_{a\in\mathcal A}
    \left(
        R^\star(x,a)-\widehat R_t(x,a)
    \right)
\le
    2
    \max_{a^2\in\mathcal A}
    \Big|
        R^\star(x,a^1)
        -
        R^\star(x,a^2)
        -
        \widehat R_t(x,a^1)
        +
        \widehat R_t(x,a^2)
    \Big|.
\end{align*}
Since this inequality holds for every $a^1$, squaring both sides
and averaging over
$a^1\sim\widehat\pi_t^1(\cdot\mid x)$ yields
\begin{align*}
    &\max_{a\in\mathcal A}
    \left(
        R^\star(x,a)-\widehat R_t(x,a)
    \right)
    -
    \min_{a\in\mathcal A}
    \left(
        R^\star(x,a)-\widehat R_t(x,a)
    \right)
    \\
    &\qquad\le
    2
    \sqrt{
        \mathbb E_{a^1\sim\widehat\pi_t^1}
        \left[
            \max_{a^2\in\mathcal A}
            \Big(
                R^\star(x,a^1)
                -
                R^\star(x,a^2)
                -
                \widehat R_t(x,a^1)
                +
                \widehat R_t(x,a^2)
            \Big)^2
        \right]
    }.
\end{align*}
Consequently,
\begin{align*}
    G_t
    &\le
    2
    \mathbb E_{x\sim d}
    \sqrt{
        \mathbb E_{a^1\sim\widehat\pi_t^1}
        \left[
            \max_{a^2\in\mathcal A}
            \Big(
                R^\star(x,a^1)
                -
                R^\star(x,a^2)
                -
                \widehat R_t(x,a^1)
                +
                \widehat R_t(x,a^2)
            \Big)^2
        \right]
    }
    \\
    &\le
    2
    \sqrt{
        \mathbb E_{\substack{
            x\sim d,\,
            a^1\sim\widehat\pi_t^1
        }}
        \left[
            \max_{a^2\in\mathcal A}
            \Big(
                R^\star(x,a^1)
                -
                R^\star(x,a^2)
                -
                \widehat R_t(x,a^1)
                +
                \widehat R_t(x,a^2)
            \Big)^2
        \right]
    },
\end{align*}
where the last inequality follows from Jensen's inequality.
Applying Eq.~\eqref{eq:bt_ucb_confidence},
\begin{align*}
    G_t
    &\le
    4\sqrt{3}\,e
    \sqrt{
        \log\frac{2N_{\mathcal R}T}{\delta}
    }
    \sqrt{
        \mathbb E_{\substack{
            x\sim d,\,
            a^1\sim\widehat\pi_t^1
        }}
        \left[
            \max_{a^2\in\mathcal A}
            \min\left\{
                1,
                U_{\operatorname{BT}}^2
                \left(
                    \lambda,x,a^1,a^2,\mathcal R;
                \mathcal D_{t-1}^{\operatorname{BT}}
                \right)
            \right\}
        \right]
    }.
\end{align*}
Summing over $t$ and applying Cauchy--Schwarz together with
Eq.~\eqref{eq:bt_ucb_eluder} gives
\begin{align*}
    \operatorname{Reg}_{\operatorname{BT}}(T)
    =
    O\left(
        \sqrt{
            T
            \left(
                d_{\operatorname{BT}}(\lambda,\mathcal R,T)
                +
                \log\frac{1}{\delta}
            \right)
            \log\frac{N_{\mathcal R}T}{\delta}
        }
    \right).
\end{align*}
Taking the minimum of the two bounds completes the proof.
\end{proof}

\section{Auxiliary Lemmas}\label{auxiliary}

\begin{lemma}[Freedman's inequality]\label{freedman}
Let $\mathcal H_i^{\mathrm{RF}}$ denote the observed history
up to round $i$.
For each $i\in[t]$, let $M_i$ be measurable with respect
to the history $\mathcal H_i^{\mathrm{RF}}$ and satisfy
$\mathbb E[M_i\mid\mathcal H_{i-1}^{\mathrm{RF}}]=0$ and $|M_i|\le b$ almost surely,
where $b>0$ is a deterministic constant.
Then, for any $\delta\in(0,e^{-2})$, with probability
at least $1-(\log_2 t)\delta$,
\[
\sum_{i=1}^t M_i
\le
4\sqrt{
    \left(
    \sum_{i=1}^t
    \operatorname{Var}
    (M_i\mid\mathcal H_{i-1}^{\mathrm{RF}})
    \right)
    \log(1/\delta)
}
+2b\log(1/\delta).
\]
\end{lemma}

\begin{lemma}[Lemma 4.2 in \citet{agarwal2012contextual}]\label{lemma6}
Fix any reward function \(R\in\mathcal R\). Let
\(x\sim d\), \(a\sim\pi(\cdot|x)\), and
$r\in[0,1]$ be a random reward satisfying
$\mathbb E[r\mid x,a]=R^\star(x,a)$. Define
$Y_R
=
(R(x,a)-r)^2
-
(R^\star(x,a)-r)^2$.
Then,
\[
\mathbb E[Y_R]
=
\mathbb E_{x\sim d,\,a\sim\pi}
\!\left[
(R(x,a)-R^\star(x,a))^2
\right],
\]
and
\[
\operatorname{Var}[Y_R]
\le
4\mathbb E[Y_R].
\]
\end{lemma}

\begin{lemma}[Lemma C.1 in \citet{zhao2025logarithmic}]
\label{lemmac1}
Let $\mathcal R$ be a finite function class mapping
$\mathcal Z$ to $\mathbb R$, with cardinality
$N_{\mathcal R}$, and suppose that $R^\star\in\mathcal R$.
Consider adaptively selected inputs $\{z_t\}_{t\ge1}$
and observations
$r_t=R^\star(z_t)+\epsilon_t$.
Conditional on the past observations
$\{(z_i,r_i)\}_{i=1}^{t-1}$ and the current input $z_t$,
assume that $\epsilon_t$ has mean zero and is
$1$-sub-Gaussian.
Define the LS estimator
$\widehat R_t\in
\arg\min_{R\in\mathcal R}
\sum_{i=1}^{t-1}\bigl(R(z_i)-r_i\bigr)^2$.
Then, for any positive integer $T$ and any
$\delta\in(0,1)$, with probability at least $1-\delta$,
the following inequality holds simultaneously
for all $t\in[T]$:
\[
\sum_{i=1}^{t-1}
\bigl(\widehat R_t(z_i)-R^\star(z_i)\bigr)^2
\le 8\log\frac{N_{\mathcal R}T}{\delta}.
\]
\end{lemma}

\begin{lemma}[Lemma 3 in \citet{wu2025greedy}]\label{lemma3wu}
Let $\mathcal P$ be a finite function class with cardinality $N_{\mathcal P}$. 
Suppose the training data
$\{(x_i,a_i^1,a_i^2,y_i)\}_{i=1}^t$
is generated according to
$y_i \sim \mathrm{Ber}\!\left(P^\star(x_i,a_i^1,a_i^2)\right)$,
where $P^\star \in \mathcal P$. Let $\widehat P_t$ be the MLE estimator over $\mathcal P$
computed from the samples $\{(x_i,a_i^1,a_i^2,y_i)\}_{i=1}^{t-1}$. Then, with probability at least $1-\delta$, simultaneously for all $t\in[T]$,
\[
\sum_{i=1}^{t-1}
\Big(
\widehat P_t(x_i,a_i^1,a_i^2)
-
P^\star(x_i,a_i^1,a_i^2)
\Big)^2
\le
2\log\frac{N_{\mathcal P}T}{\delta}.
\]
\end{lemma}

\end{document}